\documentclass[10pt,twocolumn,letterpaper]{article}

\usepackage[pagenumbers]{cvpr} 
\usepackage{hyperref} 

\usepackage{amsthm}
\usepackage{amsmath}

\usepackage{amsmath,amsfonts,amssymb}
\usepackage{bm}
\usepackage{caption}
\usepackage{booktabs} 
\usepackage{algorithm}
\usepackage{algorithmic}
\usepackage{bigfoot} 

\usepackage[table]{xcolor}
\usepackage{footmisc}

\usepackage{subcaption}

\usepackage{multirow}
\newcommand{\blfootnote}[1]{%
  \begingroup
  \renewcommand{\thefootnote}{}%
  \renewcommand{\thefootnotemark}{}%
  \footnotetext{#1}%
  \endgroup
}

\def\NN{\mathbb{N}}
\def\RR{\mathbb{R}}

\def\f{\frac}

\def\T{\mathcal{T}}

\def\C{\mathcal{C}}
\def\D{\mathcal{D}}

\def\X{\mathcal{X}}
\def\Y{\mathcal{Y}}

\def\EE{\mathbb{E}}

\def\PP{\mathbb{P}}
\def\dist{\operatorname{dist}}

\def\M{\mathcal{M}}

\def\en{\operatorname{in}}
\def\de{\operatorname{out}}
\def\app{\operatorname{bot}}
\def\proj{\operatorname{Proj}}

\def\nn{\mathcal{NN}}
\def\unet{\mathcal{U}}

\def\T{\mathcal{T}}
\def\A{\mathcal{A}}
\def\HH{\operatorname{H}}
\def\S{\mathcal{S}}

\def\error{\mathcal{E}}

\def\S{\mathcal{S}}
\def\Sc{\mathcal{S}^c}

\newcounter{JFCounter}

\newcounter{IRCounter}

\newtheorem{theorem}{Theorem}[section]
\newtheorem{definition}[theorem]{Definition}   
\newtheorem{lemma}[theorem]{Lemma}      
\newtheorem{proposition}[theorem]{Proposition}  
\newtheorem{assumption}[theorem]{Assumption}  
\newtheorem{Remark}[theorem]{Remark}

\title{CHEM: Estimating and Understanding Hallucinations in Deep Learning for Image Processing}

\providecommand{\thefootnotemark}{\thefootnote}

\author{
Jianfei Li$^{1,\dagger,*}$ \\
{\tt\small lijianfei@math.lmu.de}
\and
Ines Rosellon-Inclan$^{1,2,\dagger}$ \\
{\tt\small rosellon@math.lmu.de}
\and
Gitta Kutyniok$^{1,2,3,4}$ \\
{\tt\small kutyniok@math.lmu.de}
\and
Jean-Luc Starck$^{5,6}$ \\
{\tt\small jstarck@cea.fr}
}

\begin{document}

\maketitle

\blfootnote{\newline$^1$Ludwig-Maximilians-Universit{\"a}t M{\"u}nchen, Munich, Germany.\newline
$^2$Munich Center for Machine Learning (MCML), Germany.\newline
$^3$University of Troms{\o}, Norway.\newline
$^4$German Aerospace Center (DLR), Germany.\newline
$^5$Université Paris-Saclay, Université Paris Cité, CEA, CNRS, AIM, France.\newline
$^6$Foundation for Research and Technology Hellas (FORTH), Greece.\newline
$^*$Corresponding author.\newline
$^\dagger$Co-first author.
}

\begin{abstract}

Deep learning-based methods have recently achieved significant success in image reconstruction problems. However, challenges have emerged, as these methods may generate unrealistic artifacts or hallucinations, which can interfere with analysis in safety-critical scenarios. This paper introduces a framework for quantifying and characterizing 
hallucinated artifacts in image reconstruction models. The proposed method, termed the Conformal Hallucination Estimation Metric (CHEM), enables the identification 
of hallucination-prone regions in model predictions. It leverages wavelet and shearlet representations to localize such regions at the level of 
image features, and uses conformalized quantile regression to assess hallucination levels in a distribution-free manner. A theoretical analysis is provided, characterizing the sensitivity of CHEM to hallucinated artifacts and its relationship to the mean squared error. Building on these insights and adopting a viewpoint grounded in approximation theory, we investigate why U-shaped networks—widely used architectures for image reconstruction— tend to hallucination-prone predictions. We assess the effectiveness of the proposed approach on astronomical image deconvolution using the CANDELS dataset with architectures such as U-Net, SwinUNet, and Learnlets, and on natural image super-resolution using the DIV2K dataset with models such as DRUNet, Unfolded DRS, RAM, and DPS.
\end{abstract}



\section{Introduction}
In recent decades, artificial intelligence has permeated nearly every domain, including medical diagnosis pipelines \cite{rajeev_ranjan_kumar1_advances_2025}, autonomous driving \cite{zhang_road_2018}, natural language processing \cite{gurevych_rate_2021}, and speech recognition \cite{mehrish2023reviewdeeplearningtechniques}, among others. Despite these advances, hallucinations pose a recurring challenge that accompany the success of deep learning. The phenomenon of hallucination has been observed in a wide range of deep learning models, including Large Language Models \cite{zhang_sirens_2025}, Large Vision-Language Models \cite{liu_survey_2024}, Natural Language Generation systems \cite{ji_survey_2023}, and Large Foundation Models \cite{sahoo_comprehensive_2024}. 

\begin{figure}[]
    \centering
    \includegraphics[width=\linewidth]{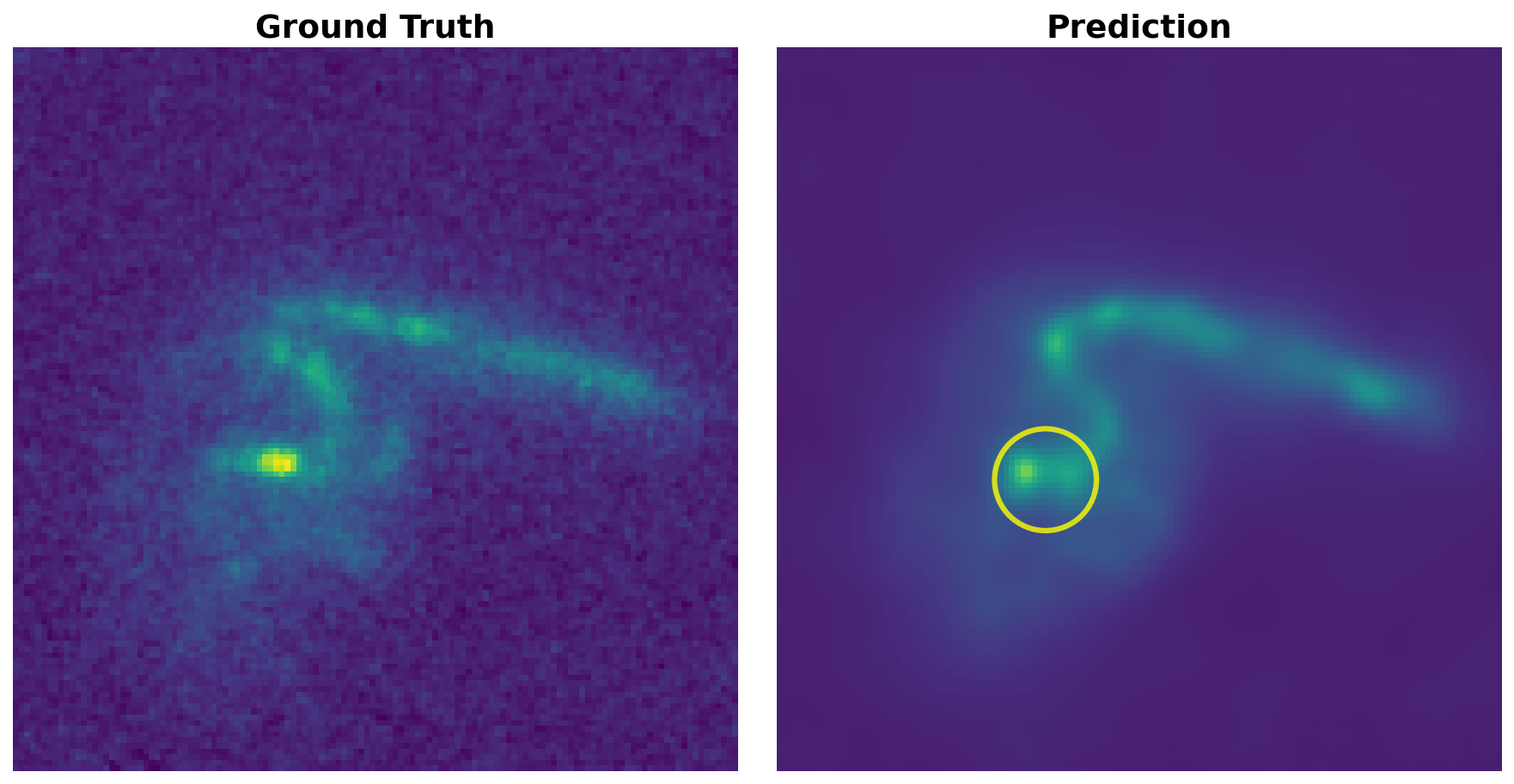}
    \caption{Hallucinations in astronomical image deconvolution: U-Net predictions 
exhibit visually plausible artifacts absent from the ground truth.
    }
    \label{fig: introductionplot}
\end{figure}
Specifically, the hallucination phenomenon has been observed in image processing \cite{gottschling_troublesome_2025-1}, raising particular concern in sensitive applications such as medical imaging \cite{bhadra_hallucinations_2021,tivnan_hallucination_2024}, as well as in scientific imaging domains including astronomy and fluorescence microscopy \cite{akhaury_ground-based_2024,belthangady_applications_2019}. Figure \ref{fig: introductionplot} illustrates that hallucinations appear in astronomical image deconvolution tasks \cite{akhaury_ground-based_2024}. Despite recent efforts to formalize hallucinations in imaging tasks ~\cite{bhadra_hallucinations_2021, tivnan_hallucination_2024}, there is no consensus on a common framework to detect and quantify them. In this work, we characterize hallucinations as image features whose 
magnitude is incompatible with noise fluctuations, and build our 
framework upon this. 
In addition to their measurement, another key challenge is to identify the underlying mechanisms that lead to such artifacts. Both understanding and estimating the confidence of the model in its outputs are crucial, as unreliable predictions may propagate through downstream pipelines and compromise decisions based on them  \cite{cohen_looks_2024}.

\subsection{Related Work}
\label{sec: Related Work}

In recent years, deep learning methods have greatly improved image reconstruction performance across a wide range of applications. However, hallucinations have increasingly been identified as a major concern when integrating deep learning into the image reconstruction pipeline \cite{9420272, bhadra_hallucinations_2021, gottschling_troublesome_2025-1, ANTUN2026303}.  We review the main approaches that have been proposed to address this challenge.

Hallucinations have been mainly studied in the context of generative models, which are widely used for image reconstruction \cite{suganthi2021review, denker2020conditional, song2021solving, moser2024diffusion}. Cohen et al.~\cite{cohen_looks_2024} characterize hallucinations as realistic-looking details that do not exist in the ground truth images— 
a notion our framework encompasses. Adopting the definition of perceptual quality as the divergence between probability densities \cite{Blau_2018}, Cohen et al.\  \cite{cohen_looks_2024} establish a fundamental trade-off in image restoration: improving perceptual quality inevitably increases uncertainty. This trade-off induces the well-known perception--distortion trade-off \cite{Blau_2018}. In their work, hallucinations are measured via entropy power, which assesses the statistical spread of a random variable. However, estimating entropy in high-dimensional spaces remains challenging.

Another approach to quantify hallucination in this setting measures the Hellinger distance between the distribution of reconstructed images and a zero-hallucination reference distribution \cite{tivnan_hallucination_2024}. In this work, the forward diffusion process is taken as the reference. This approach requires density estimation and relies on the assumption that both distributions can be well approximated by multivariate normal distributions.

A further approach \cite{ren2026hallucinationscoremitigatinghallucinations} proposes a Hallucination Score (HS) obtained by prompting a large language model (LLM) to examine predictions for possible hallucinations and assign a score between 1 and 5. To avoid the computational cost of repeatedly querying the LLM, the authors first use it to annotate a dataset and then train lighter models, such as a CNN, to predict the Hallucination Score. The resulting metric correlates well with a ``human judgment'' metric obtained by asking 11 participants to assign scores to the images. Although this metric has the advantage of incorporating semantic information aligned with human perception, the score itself may be affected by hallucinations introduced by both the LLM used for annotation and the subsequent CNN predictor, and a formal mathematical foundation is not provided.

All these methods measure hallucination at the level of an entire model. However, in many applications it is not only important to determine whether a model is prone to hallucinations, but also to localize regions in an image that pose a high risk of hallucination.
One attempt to localize hallucinations in ill-posed inverse problems is the concept of \emph{hallucination maps}, which aim to pinpoint features that are demonstrably absent from the measurements \cite{bhadra_hallucinations_2021}. This method was developed for MRI, where the imaging operator is known, and the computation of the hallucination map relies on explicit knowledge of this operator.
In galaxy image deconvolution, the \emph{hallucination rate} is defined as the difference between objects and clumps derived from SCARLET \cite{akhaury_ground-based_2024}. Detections whose centroids lie more than five pixels away from those in the ground-truth images are considered hallucinations. Although suitable for galaxy images, this method may introduce bias due to the use of SCARLET and does not account for object shape or detailed texture variations.

In other image processing tasks,  e.g. in explainability studies, hallucinations are noted in attribution masks. A \emph{Hallucination Score} quantifies them as the ratio of edges appearing solely in the explanation to those in an input image \cite{kolek_explaining_2023}. This metric is tailored for analyzing the decision-making process in image classification models rather than hallucinations in general image reconstruction tasks. 

\subsection{Contributions}
This paper focuses on providing a comprehensive analysis to assess and understand hallucination in image processing tasks. We propose a general framework that encompasses hallucinations across applications and tasks, enables their localization within an image, and can be tailored to the application domain through the choice of a representation method.
Our contributions are summarized as follows:
\begin{itemize}
    \item Hallucination quantification: We introduce a novel hallucination detection approach: \textbf{C}onformal \textbf{H}allucination \textbf{E}stimation \textbf{M}etric (CHEM). Our method is designed to highlight the hallucination-prone regions in a distribution-free manner by leveraging the expressive power of wavelets and shearlets. 

    \item Theoretical justification of CHEM: We show that CHEM captures artifacts with exponential sensitivity in the wavelet or shearlet domain and establish a formal connection between CHEM and the mean squared error, showing that controlling one has implications for the other. Building on this relationship, we employ approximation theory to analyze the reconstruction error of U-shaped networks, characterizing how architectural constraints, such as the number of parameters and the number of layers, contribute to hallucination artifacts.
    Our result also characterizes the expressivity of U-shaped networks for general image processing tasks.

    
    
    \item Experimental Results:
    We evaluate the proposed approach on two image reconstruction tasks. 
    First, we consider astronomical image deconvolution using the CANDELS dataset \cite{grogin_candels_2011,koekemoer_candels_2011} and analyze hallucination behavior in models such as U-Net \cite{ronneberger2015u}, SwinUNet \cite{fan_sunet_2022}, and Learnlets \cite{ramzi2023wavelets}. 
    Second, we study natural image super-resolution using the DIV2K dataset  \cite{8014884} across several reconstruction models, including DRUNet-Pnp \cite{zhang_plug-and-play_2021}, Unfolded DRS  \cite{Douglas1956OnTN}, RAM \cite{terris2025reconstruct}, and DPS  \cite{chung2022diffusion}. 
    Across these experiments, we further analyze the influence of different dictionaries (wavelets and shearlets), loss functions ($\ell_1$ and $\ell_2$), and training epochs. 
    Two main findings emerge: (i) there is a tradeoff between accuracy and hallucination, and (ii) models with similar accuracy may exhibit different hallucination levels.
\end{itemize}
The source code for the experiments is publicly available\footnote{
\url{https://github.com/inesrosellon/CHEM}}.

\subsection{Impact and Future Directions}
    The proposed metric CHEM is directly applicable to the astronomical imaging community, where deep learning models are increasingly used for image reconstruction (e.g., deconvolution) but are still evaluated mainly with standard performance metrics. CHEM supports a more comprehensive evaluation by localizing and measuring high-risk hallucinated artifacts within individual predictions, offering a more informative foundation for comparing models.
    
    More generally, across any safety-critical setting involving image reconstruction, CHEM provides a principled approach for thorough model assessment, training process surveillance, and for identifying predicted regions that are prone to hallucinate. This is particularly important in medical imaging, where hallucinations can be mistaken for clinically relevant anatomy, potentially leading to severe outcomes. Through its distribution-free, model-agnostic framework built on wavelet and shearlet representations together with conformalized quantile regression, CHEM delivers a methodologically robust and practically usable tool for the trustworthy AI community focused on reconstruction tasks.

 \section{Methodology}

\subsection{Notations}
Let $\RR$, $\RR_{+}$, $\NN$, and $\NN_{+}$ denote the sets of real numbers, positive numbers, nonnegative integers, and positive integers, respectively. We denote $[N]:=\{1,\dots,N\}$ if $N\in\NN_{+}$. For $x \in \RR^d$, we denote its $i$-th element as $x_i$ or $(x)_i$. We denote $\|\cdot\|_p$, $p\in[1,\infty]$ as the vector norm. Let $\Omega:= [-1,1]^d$. The $L_p(\Omega)$ space contains measurable functions that have a finite $L_p$ norm
$    \|f\|_{L_p(\Omega)} = \left( \int_{\Omega} |f(x)|^p dx \right)^{1/p} <\infty$. For $p=\infty$, we define $\|f\|_{L_\infty(\Omega)}:=\operatorname{ess\,sup}|f|$.


\subsection{Quantification of hallucination-prone regions}
In image processing tasks, we consider an operator $\mathcal{M}: \mathcal{X} \rightarrow \mathcal{Y}$ that maps an input space $\mathcal{X}$ (e.g., degraded or low-resolution images) to an output space $\mathcal{Y}$. Our goal is to learn a predictive model $\Phi(\cdot; w)$, parameterized by weights $w \in \mathbb{R}^N$, such that $\Phi(X) \approx \mathcal{M}(X)$.

To characterize the reliability of the prediction $\Phi(X)$, we construct a set-valued mapping $B_\alpha(\Phi(X)) \subset \mathcal{Y}$, representing a confidence region centered at the prediction. More precisely, for a pre-defined significance level $\alpha \in (0,1)$, we require that this region fulfill the coverage condition:
\begin{equation*}
    \mathbb{P} \left( \mathcal{M}(X) \in B_\alpha(\Phi(X)) \right) \geq 1 - \alpha.
\end{equation*}
Within this framework, a prediction that is prone to hallucinate is identified as any statistical inconsistency with the uncertainty of the model. When $\mathcal{M}(X) \notin B_\alpha(\Phi(X))$, the model yields a feature that lies outside the range of plausible reconstructions defined by its own confidence region. This indicates that the textures produced in $\Phi(X)$ are spurious artifacts rather than faithful approximations of the underlying signal.

To quantify the severity of these hallucinations, we measure the discrepancy of the ground truth from the prediction set. Let $\operatorname{dist}(\cdot, \cdot): \mathcal{Y} \times \mathcal{Y} \rightarrow \mathbb{R}_{\geq 0}$ be a metric such that $\operatorname{dist}(Y_1, Y_2) = 0$ if and only if $Y_1 = Y_2$. For any subset $A \subset \mathcal{Y}$, we denote the distance from a point $Y$ to the set $A$ as $\operatorname{dist}(Y, A) := \inf_{Z \in A} \operatorname{dist}(Y, Z)$. 

\begin{definition}[Hallucination Estimation Metric]\label{def:hallucination}
Let $\alpha \in (0,1)$ and $\mathbb{P}$ be a probability measure over the joint space $\mathcal{X} \times \mathcal{Y}$. For a neural network $\Phi(\cdot)$, we define the hallucination metric $\operatorname{H}(\Phi; \alpha)$ as the expected distance between $Y$ and the corresponding confidence region $B_\alpha(\Phi(X))$ 
\begin{equation*}
    \operatorname{H}(\Phi; \alpha) := \mathbb{E}_{(X,Y) \sim \mathbb{P}} \left[ \operatorname{dist}\left( Y, B_\alpha(\Phi(X)) \right) \right],
\end{equation*}
subject to the coverage constraint $\mathbb{P} \left( Y \in B_\alpha(\Phi(X)) \right) \geq 1 - \alpha$.
\end{definition}

In the absence of hallucinations, the coverage property ensures that the distance is zero with high probability ($1-\alpha$). Conversely, the distance becomes strictly positive, capturing the magnitude by which the model's textures deviate from the true image. In this setting, ground truth is necessary to compute the hallucination risks of $\Phi(X)$, in line with previous work \cite{bhadra_hallucinations_2021} and the formalization of hallucination in \cite{jesson_estimating_2024}. In practical scenarios, the hallucination risk of model $\Phi$ can be assessed on a test dataset, allowing us to select the model that achieves the best balance between prediction error and hallucination risk.

We emphasize that the choice of metric $\dist(\cdot,\cdot)$ directly influences how effectively we can highlight regions susceptible to hallucinations. In image processing, we are particularly interested in structured hallucinations, such as fabricated textures or objects that tend to be well-localized and anisotropic. As a result, performing hallucination detection directly in the pixel space, for example via Euclidean distance, may be suboptimal. This motivates a discussion on how to design $\dist(\cdot,\cdot)$.

\begin{Remark}
    Quantification of hallucinations is an emerging paradigm. In in-context learning with generative AI, the hallucination rate proposed in \cite{jesson_estimating_2024} is defined similarly to Definition~\ref{def:hallucination}, where $\dist(Y, B_{\alpha}(\Phi(X))):= \chi_{\{Y\notin B_{\alpha}(\Phi(X))\} }$, with $\chi_{A}(x) = 1 $ if $x$ belongs to $A$ and zero otherwise. An analogous characterization of hallucination for image processing was explored in \cite{angelopoulos2022image} and subsequently improved in \cite{kutiel_conformal_2023}, enabling the identification of reliable regions in predicted images.
\end{Remark}

\subsection{Conformal hallucination estimation metric}

In the following, we develop an algorithm to assess the hallucination estimation metric as defined in Definition~\ref{def:hallucination}. In image processing,  wavelets and shearlets have been shown as powerful dictionaries \cite{mallat1999wavelet,kutyniok2012shearlets,kutyniok2016shearlab,guo2006sparse,li2022convolutional}. Their ability to extract localized and directional texture information makes them ideal tools for highlighting hallucination-prone regions. Let $X \in \X \subset \RR^{t_1}$, $Y \in \Y \subset \RR^{t_2}$. Here, we flatten the images to vectors for simplicity. We refer to $W$ as a discrete wavelet transform (DWT) or a discrete shearlet transform (DST). For an in-depth overview of wavelets and shearlets, please refer to \cite{mallat1999wavelet}. 
The texture information is contained in $\hat X := WX \in \RR^{\hat{t}_1}$ for a certain $\hat{t}_1$ determined by $W$. 
Each $i$-th element in $\hat{X}$ contains different directional texture information in $X$. In the output space $\Y$, we can similarly compute texture information of $Y$ by a DWT/DST, denoted as $\hat Y \in \RR^{\hat{t}_2}$.

To identify areas with a high risk of hallucination in the wavelet or shearlet domain, our objective is to estimate a mapping $\hat R(X)$ such that the interval centered at $(\widehat{\Phi(X)})_j$ with radius $\hat R(X)_j$ gathers comparable information surrounding $\widehat{\Phi(X)}_j$
\begin{small}
    \begin{align*}
        B_\alpha\left(\widehat{\Phi(X)}\right)_j :=\left [\widehat{\Phi(X)}_j- \hat R(X)_j, \widehat{\Phi(X)}_j+\hat R(X)_j \right],
    \end{align*}
\end{small}
and the actual value is likely to lie within this interval with high probability
\begin{align}\label{eq:alphaball_0}
        \PP  \left( (\hat Y)_j \in B_\alpha \left(\widehat{\Phi(X)} \right)_j  \right) \geq 1-\alpha.
\end{align}
This characterizes confidence intervals with a mapping $\hat R:\RR^{t_1}\rightarrow \RR^{\hat{t}_2}$. Let us choose the metric function as $\dist(Y_1, Y_2):=\frac{1}{\hat t_2}\sum_{j=1}^{\hat t_2}|(\widehat{ Y_1})_j-(\widehat{ Y_2})_j|$ for $(Y_1,Y_2)\in \Y\times \Y$ and $B_\alpha \left(\widehat{\Phi(X)} \right) := \prod_{j=1}^{\hat t_2}B_\alpha \left(\widehat{\Phi(X)} \right)_j$. Then, Definition~\ref{def:hallucination} for measuring hallucination over a full image prediction can be equivalently rewritten by averaging texture-wise hallucination scores
\begin{align}\label{eq:ndhall}
    \operatorname{H}(\Phi):=  \frac{1}{\hat t_2}\sum_{j=1}^{\hat t_2}\operatorname{H}(\Phi)_j ,
\end{align}
where the degree of hallucination at the $j$-th directional texture information is given by
\begin{align}\label{eq:1dhall}
\begin{aligned}
    \operatorname{H}(\Phi)_j
    &=\EE  \dist\left( (\hat Y)_j, B_\alpha\left(\widehat{\Phi(X)}\right)_j  \right) \\
    &=\EE \left( |\widehat{\Phi(X)}_j - (\widehat{Y})_j | - \hat R( X)_j \right)_{+} ,
\end{aligned}
\end{align}
with $(z)_+ := \max\{z,0\}$ for any $z\in \RR$. In this context, $\alpha$ represents the level of confidence and can be extended to define the overall confidence in the image as in Definition~\ref{def:hallucination}. 

Let $\hat r(\cdot) : \RR^{t_1}\rightarrow \RR^{\hat{t}_2}$ be an initialization of $\hat R(\cdot)$.
The remaining unresolved issue is to determine how to modify an initial radius $\hat r( X)$ to reach a target radius $\hat R( X)$, while ensuring the condition \eqref{eq:alphaball_0} is met. This goal can be achieved  without strong distributional assumptions by using conformal quantile regression, as proposed in \cite{leterme2025distribution}.
Given a finite calibration dataset $\D:=\{(X_n,Y_n)\in \X\times\Y \}_{n=1}^N$. 
We define $\hat R( X)_j$ as $\hat R( X)_j := g_{\lambda_j}(\hat r( X)_j)$, where $g_{\lambda}:\RR_{+}\rightarrow\RR_{+}$ denotes a non-decreasing calibration function with $\lambda \in [a,b]$. We choose $\lambda_j$ as the $(1-\alpha)(1+1/N)$-th empirical quantile of $\{\lambda_j^n\}_{n=1}^{N}$, where
    \begin{align}\label{eq:lambda_n_j}
        \begin{aligned}
            \lambda_j^n := 
                 \inf\Big \{ \lambda \in [a,b]:  g_{\lambda}(\hat r(X_n)_j) 
                 &\geq \left | \widehat{\Phi(X_n)}_j - (\widehat{Y_n})_j \right| \Big \}.       
        \end{aligned}
    \end{align}

Here, $\lambda^n_j$ is equal to $a$ (or $b$) when the condition is always (or never) met. Let $(X,Y)$ be a new data sample. According to \cite[Proposition 1]{leterme2025distribution}, if the samples $\D\cup \{(X,Y)\}$ are independently drawn from an unknown distribution $\mu$ on $\X\times\Y$, then it holds that
\begin{align*}
    \alpha-\frac{1}{N+1}\leq\PP \left ( (\hat{Y})_j \notin B_\alpha\left(\widehat{\Phi(X)}\right)_j \right) \leq \alpha. 
\end{align*}
To evaluate the risk of hallucination in the full image, we modify Equation \eqref{eq:1dhall} by imposing an upper bound $\theta$. This ensures that the evaluation across the entire image is not disproportionately affected by the risk of hallucination at an individual texture.
We are now prepared to introduce the proposed hallucination risk score.

\begin{definition}[CHEM]\label{def:chem}
    For a trained neural network $\Phi(\cdot)$,
 the risk of hallucination measured in the wavelet or shearlet domain is defined by $\operatorname{H}^\theta(\Phi)_j$ and $\operatorname{H}^\theta(\Phi)$
    \begin{align*}
    \quad\operatorname{H}^\theta(\Phi)_j
        &:=\EE H^\theta ( X, Y)_j, \,\,
        \operatorname{H}^\theta(\Phi) := \frac{1}{\hat t_2}\sum_{j=1}^{\hat t_2}\operatorname{H}^\theta(\Phi)_j, 
    \end{align*}
    where
    \begin{align*}
        H^{\theta}( X, Y)_j := \min\left \{\left( \left|\widehat{\Phi(X)}_j - (\hat Y)_j\right | - \hat R(X)_j\right)_{+},\theta\right \},
    \label{eq:overall-texture-level-hallucination}
    \end{align*}
    for some $\theta>0$. We refer to these quantities as the \textbf{C}onformal \textbf{H}allucination \textbf{E}stimation \textbf{M}etric (CHEM). 
\end{definition}

In Figure~\ref{fig:cross-dict-hallucination}, we illustrate the effect of combining CHEM with different dictionaries, specifically Haar, db4, and db8 wavelets~\cite{daubechies_ten_1992}, and shearlets. The following result shows that for application purposes it suffices to evaluate CHEM using finitely many samples. The algorithm \ref{alg:CHEM} summarizes the procedure for computing CHEM.
\begin{proposition}\label{prop:hoeffding}
    Let $\theta \in \RR_{+}$ and $\D\cup\{ (X_m,Y_m)\in \X\times\Y  \}_{m=1}^M$ be independent and identically distributed random variables. Then, with probability at least $1-\delta$, the following inequality holds
    \begin{small}
        \begin{align}
          \left|\operatorname{H}^\theta(\Phi)_j  - \frac{1}{M} \sum_{m=1}^M H^\theta(X_m,Y_m)_j \right| \leq \frac{  \sqrt{\theta^2\log(2/\delta)}}{\sqrt{2M}}.
    \end{align}
    \end{small}
    In addition, under the same condition, the inequality
    \begin{small}
        \begin{align}
          \left|\operatorname{H}^\theta(\Phi)  - \frac{1}{M} \sum_{m=1}^M \frac{1}{\hat t_2}\sum_{j=1}^{\hat t_2} H^\theta(X_m,Y_m)_j \right| \leq \frac{  \sqrt{\theta^2\log(2/\delta)}}{\sqrt{2M}}.
    \end{align}
    \end{small}
    holds with probability at least $1-\delta$.
\end{proposition}
We defer the proof to Section \ref{proof:hoeffding}.

\begin{Remark}
\emph{
Conformal prediction has been gaining considerable attention in the field
of uncertainty quantification \cite{teneggi2025conformal, kutiel_conformal_2023}.
To construct $B_{\alpha}(\cdot)$, we build on conformal quantile regression
(CQR), introduced in \cite{romano2019conformalized} for the one-dimensional
setting and extended to inverse problems in \cite{leterme2025distribution}.
An alternative framework for achieving the statistical guarantees of
$B_{\alpha}(\cdot)$ would be risk-controlling prediction sets (RCPS)
\cite{angelopoulos2022image}; however, as noted in
\cite{leterme2025distribution}, CQR typically yields sharper uncertainty
estimates. In \cite{leterme2025distribution} a higher uncertainty is related to high-density regions via the distance between predictions and intervals. We extend this direction by enabling
quantification in the wavelet/shearlet domain.
}
\end{Remark}

\begin{algorithm}[ht]
\caption{Conformal Hallucination Estimation Metric}
\label{alg:CHEM}
\begin{algorithmic}[1] 
\REQUIRE {Test set $\mathcal{D}_1 = \{(X_m,Y_m)\}_{m=1}^M$, calibration set $\mathcal{D}_2 =\{(X_n,Y_n)\}_{n=1}^N$, network $\Phi$, $g_\lambda$,  DWT/DST $W$, confidence level $\alpha$, radius $\hat{r}$, threshold $\theta$}; 


\FORALL {$j=1,\dots,\hat t_2$}
\FORALL {$n=1,\dots,N$}
    \STATE Solve $\lambda_j^n$ in \eqref{eq:lambda_n_j}
\ENDFOR
    \STATE $\lambda_j \leftarrow (1-\alpha)(1+1/N)$-th empirical quantile of $\{\lambda^n_j\}_{n=1}^N$
    \STATE $\hat R(\cdot)_j \leftarrow g_{\lambda_j}(\hat r(\cdot)_j)$
\ENDFOR
\FORALL {$m=1,\dots,M$}
        \STATE Compute $H^{\theta}( X_m, Y_m)_j$ using \eqref{eq:overall-texture-level-hallucination}
\ENDFOR

\STATE $\operatorname{ H}^\theta(\Phi) \leftarrow \frac{1}{M} \sum_{m=1}^M \frac{1}{\hat t_2}\sum_{j=1}^{\hat t_2} H^\theta(X_m,Y_m)_j $

\STATE \textbf{Output:} CHEM $\operatorname{ H}^\theta(\Phi)$
\end{algorithmic}
\end{algorithm}

\subsection{Theoretical properties of CHEM}
In what follows, we provide a rigorous theoretical analysis of CHEM,
establishing its mathematical foundations and characterizing how it
leverages the geometric structure of hallucinations. We first establish its detection sensitivity and then characterize its relationship with standard reconstruction errors.

\subsubsection{Detection sensitivity of artifacts}

While CHEM offers a flexible and distribution-free framework for quantifying hallucination-prone regions, its practical effectiveness fundamentally depends on the interaction between the selected transform $W$ and the conformal quantile regression. By projecting reconstruction residuals into a representation that promotes sparsity in the image features, CHEM functions as a geometric magnifier, enabling the efficient capture of structured artifacts. 
To formulate the analysis, we decompose the reconstructed image as $\Phi(X):=Y+h+\eta$ where $h$ represents a hallucination artifact and $\eta$ is a sub-Gaussian residual that could arise from input noise.
This sub-Gaussian assumption is grounded in its ability to model common distributions like Gaussian or bounded errors, and mathematically essential as it provides the concentration of measure required to bound the hallucination detection failure probability in the following result.
\begin{theorem}[Detection Sensitivity]\label{thm:chem_sensitivity}
	Let $Y$ be the ground truth image and $\Phi(X)$ be the model reconstruction. Assume that the prediction contains a hallucination artifact $h$ and a stochastic residual $\eta$, such that $\Phi(X) = Y + h + \eta$ with $ \eta$ being sub-Gaussian. Let $W$ be a Parseval frame  (e.g., shearlets or wavelets) in which $h$ is $s$-sparse. For a significance level $\alpha \in (0,1)$ and the corresponding calibrated radius function $\hat{R}$, the probability that CHEM fails to detect hallucination $h$ (i.e., the ground truth $Y$ is contained within the prediction set $B_\alpha(\Phi(X))$) is bounded by 
	\begin{align*}
		\mathbb{P}\left( \frac{1}{\hat t_2}\sum_{j=1}^{\hat t_2}\operatorname{dist}( (\hat Y)_j, B_\alpha(\widehat{\Phi(X)})_j) = 0 \mid h \right) \\
        \leq 2\exp \left( - \frac{c\left( s^{-1/2} \|h\|_2 - \hat{R}(X)_{j^*} \right)^2}{\|\eta\|_{\psi_2}^2} \right),
	\end{align*}
	where $s$ is the sparsity level of $h$ in the $W$-domain, $j^* = \arg \max_j |(\hat h)_j|$, $\|\eta\|_{\psi_2}$ denotes the sub-Gaussian norm of $\eta$, and $c$ denotes an absolute constant arising from the sub-Gaussian tail.
\end{theorem}


The proof is postponed to Section \ref{proof:chem_sensitivity}. Theorem~\ref{thm:chem_sensitivity} mathematically formalizes the geometric amplifier effect of the transform $W$. By concentrating the energy of the structured artifacts in a minimal set of coefficients $s \ll \hat{t}_2$, the resulting factor $s^{-1/2}$ implies that structured hallucinations are easier to detect than diffused errors.
This is consistent with our characterization of hallucinations as image 
features incompatible with noise fluctuation. A critical implication of this scaling is that, as the artifact becomes more structured, i.e., $s^{-1/2}\|h\|_2 \gg \hat{R}(X)_{j^*}$, the probability that CHEM successfully captures the artifact $h$ increases. In addition, as the magnitude of hallucination $\|h\|_2 \to \infty$, the probability of not detecting the hallucination decays exponentially. 
Furthermore, because this sensitivity bound is agnostic to the specific architecture of the $\Phi$ model, CHEM serves as a universal measurement of trustworthiness across various frameworks. 

\subsection{Relationship between CHEM and MSE}

Theorem~\ref{thm:chem_sensitivity} establishes that hallucinations are detectable, we now analyze how CHEM differentiates between reconstruction error and artifacts. Unlike the mean squared error (MSE), which treats all residuals equally, CHEM can leverage the multi-scale nature of wavelet and shearlet representations to partition the coefficient index set into hallucination-free and hallucination-prone components.

Let $\S \subset [\hat{t}_2]$ and $\Sc $ be the complement of $\S$. We decompose
$$\HH(\Phi)  = \frac{1}{\hat t_2}\sum_{j=1}^{\hat t_2} \EE  \dist\left( (\hat Y)_j, B_\alpha\left(\widehat{\Phi(X)}\right)_j  \right)$$ into two spectral components 
\begin{align*}
    \HH_{\S} &:= \frac{1}{\hat{t}_2} \sum_{j \in \S} \mathbb{E} \left| (\hat{Y})_j - \proj_{B_{\alpha, j}}((\hat{Y})_j) \right|, \\
    \HH_{\Sc} &:= \frac{1}{\hat{t}_2} \sum_{j \in \Sc} \mathbb{E} \left| (\hat{Y})_j - \proj_{B_{\alpha, j}}((\hat{Y})_j) \right|,
\end{align*}
where $\proj_{B_{\alpha, j}}$ denotes the projection onto the conformal prediction set $B_{\alpha}(\widehat{\Phi(X)})_j$ at the $j$-th index. Similarly, we use $\proj$ for the collection of projections of all indices.

Denote the error over $\S$ as
\[
\error_{\S}(\widehat{\Phi(X)}, \hat{Y}) := \frac{1}{\hat{t}_2} \sum_{j \in \S} \mathbb{E}\big|(\hat{Y})_j - (\widehat{\Phi(X)})_j\big|.
\] 
We define the hallucination ratio over $\S$ as
\begin{equation*}
    R_{\S} := \frac{\error_{\S}( \widehat{\Phi(X)}, \hat{Y})}{\error( \widehat{\Phi(X)}, \hat{Y})}.
\end{equation*}
Here, $\error$ is the error across all indices in $[\hat t_j]$.

To characterize the behavior of CHEM in the presence of hallucinations, 
we introduce an assumption on textures with high hallucination risk.

\begin{assumption}\label{ass:highfreq_hall}
Let $(X, Y)$ be a new data pair. We assume $\Phi(X)$ contains hallucinations while maintaining a relatively accurate reconstruction. Denote the hallucination-prone region as $\S$ and the hallucination-free region as $\Sc$. Mathematically, this means that, on $S$, the projection of the ground truth onto the prediction interval lies substantially closer to the prediction than the true value itself, whereas on $\Sc$ the projection of the ground truth preserves the original information, i.e.,
\begin{align*}
    \error_{\S} (\widehat{\Phi(X)} , \proj(\hat{Y})) &\leq \mu \error_{\S} (\widehat{\Phi(X)} , \hat{Y}) ,
\end{align*}
and
\begin{align*}
    \error_{\Sc} (\widehat{\Phi(X)} , \proj(\hat{Y})) &= \nu \error_{\Sc} (\widehat{\Phi(X)} , \hat{Y}),
\end{align*}
where $0 \leq \mu \ll 1$ and $\nu \approx 1$.
\end{assumption}

\begin{theorem}\label{thm:hall}
Under Assumption~\ref{ass:highfreq_hall}, the components of CHEM satisfy:
\begin{align}\label{thm: high frequency}
\begin{aligned}
    \HH_{\S} &\geq (1-\mu) R_{\S} \error(\hat{Y} , \widehat{\Phi(X)}) \text{ and}  \\
    \HH_{\Sc} &\leq (1+\nu)(1-R_{\S}) \error(\hat{Y} , \widehat{\Phi(X)}).
\end{aligned}
\end{align}
\end{theorem}

The proof is given in Section \ref{proof:hall}. The analysis of Theorem~\ref{thm:hall} establishes a framework for detecting model-generated artifacts by decomposing the coefficient space into high-risk (${\S}$) and safe (${\Sc}$) bands. The theorem demonstrates that CHEM is a faithful detector through $\HH_{\S} \geq (1-\mu) R_{\S} \error(\cdot)$, which guaranties that when a model is ``confidently wrong''—producing prediction intervals far from the ground truth—the metric scales linearly with the actual hallucination magnitude and cannot be zero.  Simultaneously, the hallucination-free upper bound, $\HH_{\Sc} \leq (1+\nu)(1-R_{\S}) \error(\cdot)$, ensures specificity by proving that CHEM ignores the reconstruction noise where conformal coverage is maintained, thus avoiding false alarms.

Theorem~\ref{thm:hall} further demonstrates that CHEM is explicitly linked to the overall reconstruction error $\mathcal{E}$. This, in turn, highlights the importance of determining how network parameters, particularly depth, neuron count, and input image size, control the formation of these artifacts. We investigate this in the following section.

\section{Understanding the Causes of Hallucinations}

To gain a deeper understanding of the origins of hallucinations, we conduct a more detailed discussion of the error analysis, which is tightly associated with hallucination-prone regions (Theorem~\ref{thm:hall}). This investigation is based on the analysis of the error bounds of deep learning methods for imaging tasks.
Since images can be regarded as discrete representations of real-world scenes \cite{barrett2013foundations}, in this section, we consider learning a general mapping $\M : \C(\Omega) \rightarrow \C(\Omega)$ with neural networks where $\C(\Omega)$ contains all continuous functions. Therefore, the discussion will include an examination of the discretization error to provide a better understanding of hallucination, which is more directly related to real-world scenes. 

Let $\xi := \{ \xi_1, \xi_2,\dots, \xi_t \} \subset \Omega$ be a set of sample locations. 
We denote by $S(f,\xi)$ a sample vector (or a discrete image) of a function (or a continuous scene) $f\in \C(\Omega)$ on $\xi$, i.e., $S(f,\xi) = (f(\xi_1), f(\xi_2), \dots, f(\xi_t))^\top$. We denote $\Pi_m$ as the collection of all polynomials in $d$ variables of coordinate-wise degree no more than $m$. For error analysis, we require the modulus of continuity of $f \in \C(\Omega)$, which is given by
\begin{align}
    \omega_{f}(r;\Omega) := \sup_{\substack{x_1,x_2 \in \Omega,\\ \|x_1 - x_2\|_{2} \leq r}} |f(x_1) - f(x_2) |.
\end{align}

For the imaging learning task $\M$, we focus on U-shaped network architectures. The original U-Net proposed by Ronneberger et al. (2015) \cite{ronneberger2015u} combines a contracting encoder, an expansive decoder, and skip connections, and has driven major advances in applications such as medical image analysis \cite{isensee_nnu-net_2021} and autonomous driving \cite{zhang_road_2018}. This architecture has, in turn, motivated numerous variants \cite{zhou_unet_2018,qin_u2-net_2020,liu_multi-level_2018,chen_transunet_2021,cao_swin-unet_2023,ma_u-mamba_2024,chen_encoder-decoder_2018}. Throughout this section, we denote by $\unet(L,K)$ a U-shaped network with at most $L$ layers and no more than $K$ trainable parameters. For ease of exposition, the precise definition is postponed to the appendix. Figure~\ref{fig:UshapedNet} depicts the generic structure of the U-shaped architecture.

 To establish non-asymptotic error bounds, we introduce the following two assumptions.

\begin{assumption}\label{ass:Lipschitz}
    We assume that the mapping $\M$ satisfies the Lipschitz property, i.e., for any $f_1,f_2\in \C(\Omega)$, $\|\M(f_1) - \M(f_2) \|_{L_\infty} \leq L_{\M} \|f_1 - f_2\|_{L_\infty}$
    for some constant $L_{\M}$.
\end{assumption}

\begin{assumption}\label{ass:Y}
    We assume that any function $g$ in the output space $\Y \subset \C(\Omega)$ satisfies $|g(x_1) - g(x_2)| \leq L_{\Y} \|x_1 - x_2\|_2$,
    for any $x_1,x_2 \in \Omega$
    and some constant $L_{\Y}$.
\end{assumption}

Given that standard neural network architectures, composed of Lipschitz activation functions and bounded linear layers, are themselves Lipschitz continuous, it is natural to frame the learning objective as identifying the optimal Lipschitz mapping within a hypothesis space. We define $\M$ as
\begin{align*}
    \M \approx \operatorname{arg\,min}_{m \in \text{Lip}} \mathbb{E}_{(f,g) \sim \mathbb{P}} \operatorname{Loss}\left( m(f) , g \right),
\end{align*}
where loss function can be chosen as $L_2$ or $L_\infty$ norm. While the Lipschitz property is violated in the context of ill-posed inverse problems, practical solutions often rely on regularization (e.g., Tikhonov or sparsity constraints) to restore stability. Under these regularized regimes, the resulting mapping becomes stable and unique. Consequently, we adopt the Lipschitz assumption to facilitate our theoretical analysis of the learned model, acknowledging that extensions to more general, non-Lipschitz operators remain a subject for future research.

The following theorem characterizes the upper bound for approximating $\M$ by U-shaped neural networks.


\begin{theorem}\label{thm:main}
    Let $d, m, L, K \in \NN_{+}$, $t=(m+1)^d$ and $\Omega = [-1,1]^d$. Let $\X \subset \{ f\in\C(\Omega): \|f\|_{L_\infty} \leq 1\}$. Suppose Assumption~\ref{ass:Lipschitz} and Assumption~\ref{ass:Y} hold. Then, there exists a linear mapping $V_m :\C(\Omega)\rightarrow \Pi_m$ and a set of points $\xi_{\en} \subset \Omega$ with $|\xi_{\en}|=t$ such that for any $\xi_{\de} \subset \Omega$ with $|\xi_{\de}|=t$, we can find a U-shaped network $\Phi \in \unet\left(O( L\log (t K)), O(t^2 K)\right)$, satisfying 
    \begin{small}
        \begin{align}
        &\left\|S(\M(f),\xi_{\de}) - \Phi(S(V_m (f), \xi_{\en}))  \right\|_{\infty} \\
        &\leq C_1 \omega_{f}\left(\frac{2}{m} \right) + \frac{C_2}{m} +  \frac{C_3 t^{3} }{  ( L K \log(K/L))^{1/t}} ,
    \end{align}
    \end{small}
    for any $f\in \X$. Here, $C_1$,$C_2$,$C_3>0 $ depend on $d, L_{\M}, L_{\Y}$.
\end{theorem}

We defer the proof to Section \ref{proof:main}. In practical scenarios, our estimation must rely on discrete image sets. Consequently, in Theorem~\ref{thm:main}, the input image $X = S(V_m (f), \xi_{\en})$ is discretized from $f$ by using $V_m$ and $\xi_{\en}$ , while the true image $Y = S(\M(f),\xi_{\de})$ is considered as a sampled image from $\M(f)$. Each $X$ and $Y$ can be interpreted as a $d$-dimensional tensor, with each dimension having a size of $m+1$. Therefore, $t$ denotes the number of elements within $X$ or $Y$.\\
The upper bound can be divided into two components. The first component is given by $C_1 \omega_{f}\left(\frac{2}{m} \right) + \frac{C_2}{m}$. A large value of $\omega_f(\cdot)$ indicates that $f$ lacks smoothness. 
However, even for highly regular functions, this term cannot be faster than 
$O(m^{-1})$. 
Essentially, this component reflects that both the inherent complexity of real scenes and the discretization of $f$ may lead to hallucinations.\\
The second part, $C_3 t^{3} (L K \log(K/L))^{-1/t}$, implies that with a sufficient number of neurons and layers, this component can be made arbitrarily small. However, due to the finite number of image pairs on which neural networks can be trained, there is a risk of overfitting. Additionally, as the dimension $d$ approaches infinity, this term weakens rapidly, indicating that hallucination may occur more often in higher-dimensional scenarios if $L$ and $K$ remain constant.

By combining Theorem \ref{thm:hall} with Theorem~\ref{thm:main}, we can bound the stable component of CHEM as:
	\begin{equation*}
		\operatorname{H}_{\mathcal{S}^c} \lesssim (1-R_{\mathcal{S}}) \left( \omega_f\left(\frac{2}{m}\right) + \frac{1}{m} + \frac{ t^{3} }{ ( L K \log\frac{K}{L})^{\frac{1}{t}}} \right).
	\end{equation*}
The subsequent theorem establishes the existence of a lower bound with a similar expression, indicating that it is unreasonable to expect $\operatorname{H}_{\S}$ to be small for every possible input.


\begin{theorem}\label{thm:lowerbound}
    Let $d, m \in \NN_{+}$ and $t=(m+1)^d$. Let $\M:\Pi_m \rightarrow \Pi_m$ and $\X =  \Pi_m$. Suppose that Assumption~\ref{ass:Lipschitz} holds. Then there exists $\xi_{\en}, \xi_{\de} \subset [-1,1]^d$ with $|\xi_{\en}| = |\xi_{\de}|=t$ and a constant $C>0$ such that 
        \begin{align*}
            \begin{aligned}
                &\inf_{\Phi\in \unet(L, K) } \sup_{f \in \X} \left\|S(\M(f),\xi_{\de}) - \Phi(S(f, \xi_{\en}))  \right\|_{\infty} \\
            &\geq C (tK^2L\log (tK^2))^{-1/t}.
            \end{aligned}
        \end{align*}
\end{theorem}
The proof is given in Section \ref{proof:lowerbound}.

\section{Experiments}
\label{sec: ExperimentalSetupAstro}

In our experiments, we evaluate CHEM across different models for two reconstruction tasks: astronomical image deconvolution and natural image super-resolution. The overarching goal is twofold: (i) to assess the hallucination behavior of different reconstruction models, and (ii) to provide visual evidence of the regions where such hallucinations occur.


\subsection{Image deconvolution on astronomical images}
We begin with astronomical image deconvolution, a setting in which
models are trained from scratch on galaxy images from the CANDELS dataset.
\subsubsection{Experimental setup}
\paragraph{Dataset}
The Cosmic Assembly Near-IR Deep Extragalactic Legacy Survey (CANDELS) \cite{grogin_candels_2011,koekemoer_candels_2011} dataset contains observations of more than 250,000 galaxies. Data extraction follows the procedure described in \cite{akhaury_deep_2022}, resulting in an initial dataset of 20,085 galaxy cutouts. From this dataset, we use 10,000 cutouts for model training, of which 70\% are used for training and 30\% for validation at each epoch. 

For the computation of CHEM, the remaining samples are used to construct the calibration dataset. These samples are evenly divided into two subsets: $\mathcal{D}_1$, used to initialize $\hat r(\cdot)$ (see \cite[Section 3]{romano2019conformalized}), and $\mathcal{D}_2$, used to compute $\hat R(\cdot)$.  Performance is evaluated on a test set containing 2,232 galaxy images.

\paragraph{Image deconvolution}
The task of reconstructing high-quality ground-based all-sky images,  naturally leads to an image deconvolution problem, i.e. the task of recovering a ground truth image \( x \) from an observed image \( y \), modeled as
\begin{equation}
y = h * x + \eta ,
\label{eq:deconv}
\end{equation}
where \(h\) denotes the  point spread functions (PSF), \(\eta\) represents additive Gaussian noise, and \(*\) denotes the two-dimensional convolution operator. Figure~\ref{fig:DataPipeline2} illustrates the image deconvolution problem on the CANDELS dataset.
\begin{figure}[http]
    \centering
    \includegraphics[width=\linewidth]{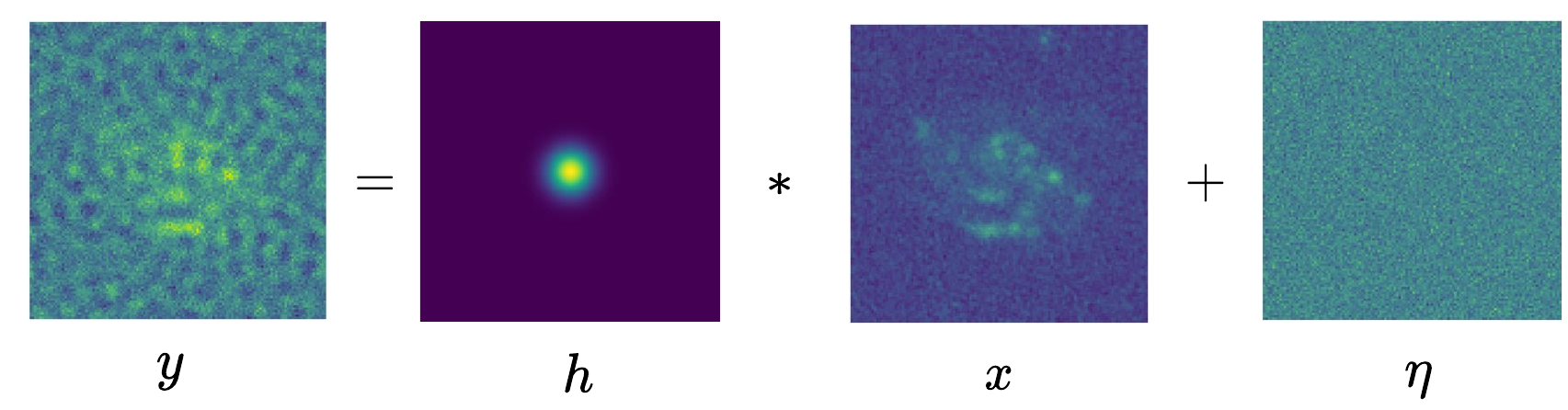}
    \caption{Image deconvolution task on the CANDELS dataset.}
    \label{fig:DataPipeline2}
\end{figure} \\

\paragraph{Models}
For the astronomical experiments we consider variants of \emph{Tikhonet}, a two-step deconvolution method ~\cite{sureau_deep_2020}. The first step applies Tikhonov deconvolution, followed by a neural network–based denoising step.
Recent works apply a variety of architectures for the denoising step, such as the U-Net–based approach~\cite{sureau_deep_2020}, the Learnlet-based method~\cite{akhaury_deep_2022}, and the Swin Transformer U-Net~\cite{akhaury_ground-based_2024}.

Tikhonov deconvolution uses a quadratic-regularized inverse filter.  
Let \( H \) denote the circulant matrix associated with the convolution operator \( h \).  
Then, the deconvolution problem $y = h *x + \eta$ is solved via
\[
  \hat{x} = (H^\top H + \lambda\, \Gamma^\top \Gamma)^{-1} H^\top y,
\]
where \( \Gamma \) is a linear Tikhonov filter, typically chosen as a Laplacian high-pass filter to penalize high-frequency components, and \( \lambda \in \mathbb{R}_+ \) is a regularization parameter selected by minimizing Stein's Unbiased Risk Estimate (SURE). We refer to ~\cite{akhaury_deep_2022} for further details.

We consider all three alternatives for the denoising stage.
The U-Net architecture, was introduced by Ronneberger et al.\ (2015), and comprises a contracting encoder for hierarchical feature extraction and an expansive decoder for resolution recovery, coupled with skip connections\cite{ronneberger2015u}.
Learnlets follows the general encoder–decoder principle of U-Net but replaces its design with an analysis–synthesis framework based on learned filter banks. The analysis stage can be interpreted as a wavelet transform with learned filters, while the synthesis stage corresponds to a wavelet reconstruction operator with learned filters. The architecture was originally proposed for denoising and was shown to have better generalization properties than U-Net \cite{ramzi2023wavelets}. 
Swin Transformers \cite{9710580} enabled the integration of Transformers, which had previously been highly successful in natural language processing \cite{huang_advancing_2024}, into vision tasks. A subsequent work incorporated Swin Transformer blocks into the U-Net architecture, as illustrated in Figure~\ref{fig: tikhonet}, and demonstrated that this design achieved performance competitive with existing benchmarks in image denoising \cite{fan_sunet_2022}.
\begin{figure}[h!]
    \centering
    \includegraphics[width=\linewidth, keepaspectratio]{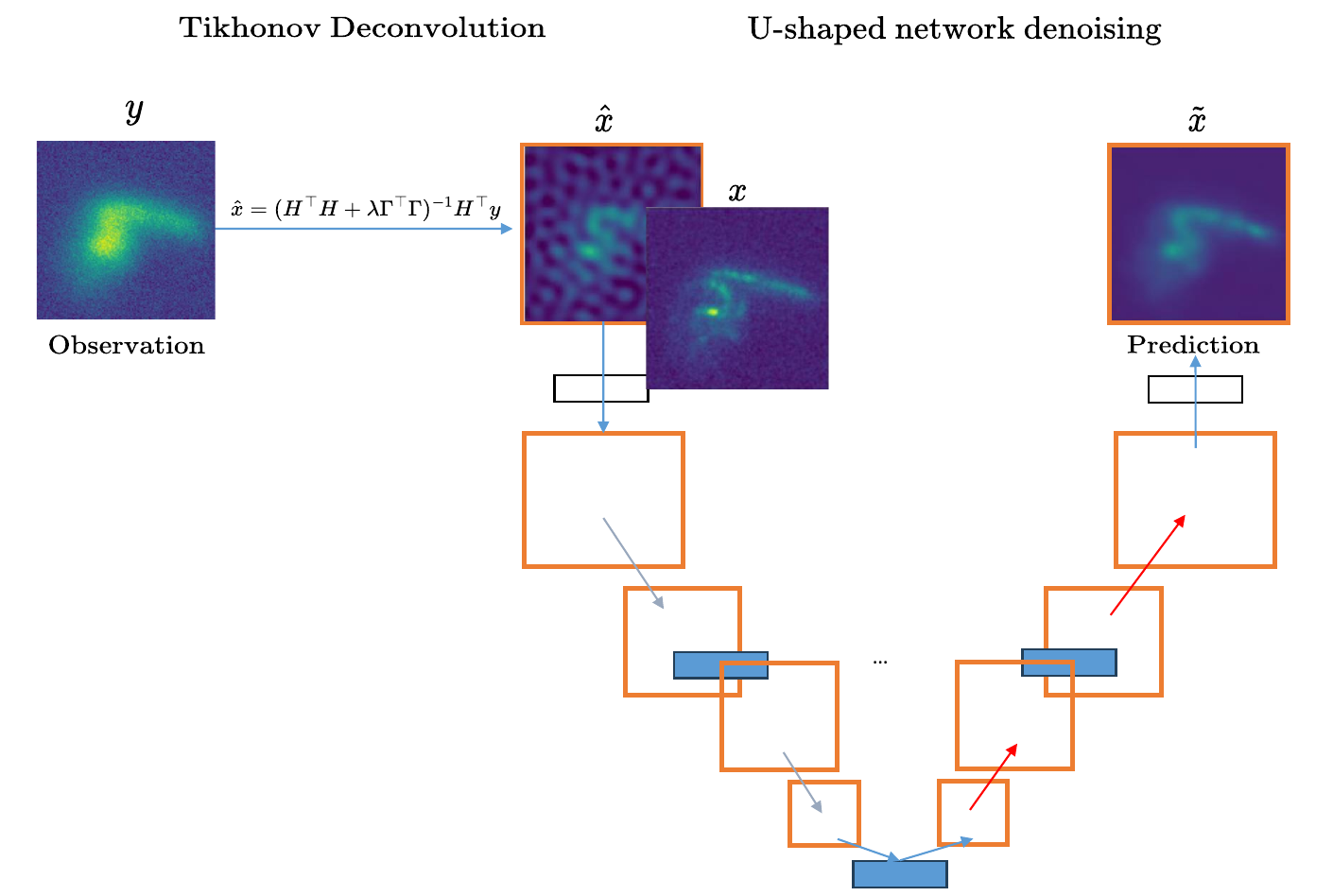}
    \caption{Example: Tikhonet with a U-shaped denoising module. For SUNet, blue filled rectangles indicate Swin-Transformer blocks, red arrows denote dual up-sampling, blue arrows denote patch-merging, and black outlined rectangles correspond to 3×3 convolutions.}
    \label{fig: tikhonet}
\end{figure}  

These three configurations correspond to large, medium, and lightweight architectures, as shown in Table \ref{tab:training-configs}, allowing us to study hallucination behavior across different model capacities.\\
\paragraph{Training setting}
In order to enable a fair comparison across architectures, we trained all networks using the same set of parameters. The only variation across experiments was the choice of loss function, where we considered either \(\ell_1\) or \(\ell_2\) loss depending on the configuration. Each model is trained for 500 epochs with a batch size of 4. The learning rate is adjusted using a two-stage scheduling strategy that combines a linear warm-up and cosine annealing. During the first three epochs, the learning rate gradually increases from zero to the initial value of \(2\times10^{-4}\) (warm-up phase). Subsequently, a cosine annealing scheduler progressively reduces the learning rate to a minimum value of \(1\times10^{-6}\) over the remaining training epochs. The network weights corresponding to the epoch with the highest validation Peak Signal-to-Noise Ratio (PSNR) were saved and used for all subsequent analysis.
\begin{table}[h]
    \centering
    \scriptsize
    \begin{tabular}{lccccc}
    \toprule
    \textbf{Method} & \textbf{Loss} & \textbf{No.\ of } & \textbf{Batch size} & \textbf{Epochs} & \textbf{Training } \\
     &  & \textbf{parameters} &  &  & \textbf{Time [h]} \\
    \midrule
    \multirow{2}{*}{Learnlets} & L1 & \multirow{2}{*}{21,673} & \multirow{2}{*}{4} & \multirow{2}{*}{500} & 65.9 \\
    & L2 &  &  &  & 61.0 \\
    \hline
    \multirow{2}{*}{SUNet} & L1 & \multirow{2}{*}{99,475,367} &    \multirow{2}{*}{4} & \multirow{2}{*}{500}& 134.0 \\
    & L2 &  &  &  & 135.1 \\
    \hline
    \multirow{2}{*}{U-Net} & L1 & \multirow{2}{*}{7,781,761} &  \multirow{2}{*}{4} & \multirow{2}{*}{500} & 75.0 \\
     & L2 &  &  &  & 76.4 \\
    \bottomrule
    \end{tabular}
    \caption{Training configurations.}
    \label{tab:training-configs}
\end{table}

\paragraph{Evaluation metrics and CHEM computation}
\label{sec:EvMetrics}
We compare reconstruction accuracy, measured by the mean squared error (MSE), and hallucination, quantified by CHEM. CHEM is computed following Algorithm \ref{alg:CHEM} with $\alpha = 0.01$ and $\theta = 1$.\\
For the hallucination maps shown in Figures~\ref{fig:cross-model-hallucination}, and \ref{fig:cross-dict-hallucination} the CHEM values associated with each coefficient \(j\) were further standardized according to
\[
\frac{\operatorname{H}^{\theta}(\Phi)_j - \operatorname{H}^{\theta}(\Phi)}{\sigma_{\operatorname{H}^{\theta}(\Phi)}},
\]
where \(\operatorname{H}^{\theta}(\Phi)\) and \(\sigma_{\operatorname{H}^{\theta}(\Phi)}\) denote the mean and standard deviation of hallucination values, respectively. 
Coefficients at fine scales whose normalized value exceeds a threshold of \(0.5\) are selected and used to reconstruct the corresponding image components. Empirically, restricting the reconstruction to the two finest scales  yields the sharpest localization of hallucinations.

\subsubsection{Model comparison}
 Figure~\ref{fig:cross-model-hallucination} illustrates the quantification of hallucinations using various U-shaped architectures. The U-Net trained with $\ell_2$ loss, displays distinct hallucinations, which are effectively identified by our CHEM method. Conversely, SUNet and Learnlets do not exhibit noticeable hallucinations in their predicted images, consistent with their corresponding MSE-CHEM/FWHM curves in Figure \ref{fig:FWHM-analysis}. 
 \begin{figure*}[ht]
    \centering
    \includegraphics[width=\linewidth]{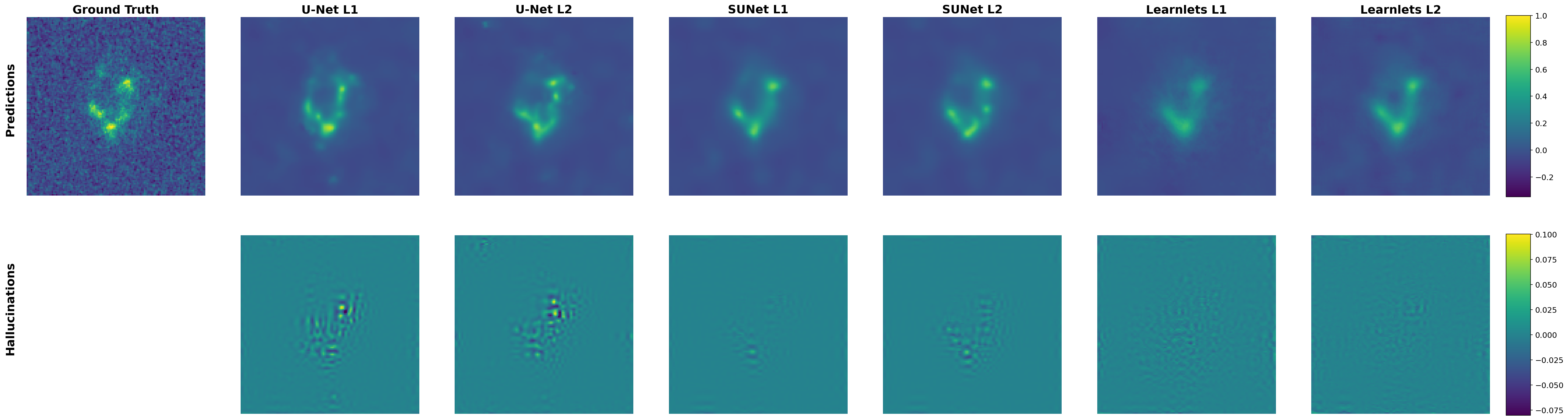}
    \caption{Quantifying hallucinations of U-shaped networks trained with different loss functions using db8. The predicted images are gathered in the first row, while the second row displays CHEM concerning $\operatorname{H}^{\theta}(\Phi)_j$ for high-resolution coefficients relative to $j$. The U-Net trained with $\ell_2$ loss exhibits distinct texture hallucinations, which are identified by CHEM. In contrast, SUNet and Learnlets do not display noticeable hallucinations, consistent with the results in Figures \ref{fig:Varying Dictionaries: Pyramids}, \ref{fig:FWHM-analysis}, \ref{fig:MSE vs Hallucination Index}. For more information, please see Section \ref{sec:EvMetrics}. }
    \label{fig:cross-model-hallucination}
\end{figure*}
\subsubsection{Hallucinations under input perturbation}
In the deconvolution problem, the point spread functions PSF is assumed to be known.
In real applications, the PSF needs to be reconstructed and the use of different instruments for measurements can result in variations of the PSF \cite{starck_astronomical_2006}. In our setting, we simulate these differences by convolving the ground truth images with a normalized Gaussian PSF of varying full width at half maximum (FWHM). 
The larger FWHM values correspond to stronger blurring, whereas the smaller FWHM values indicate sharper imaging performance.
\begin{figure}[h!]
    \centering
        \centering
        \includegraphics[width=\linewidth]{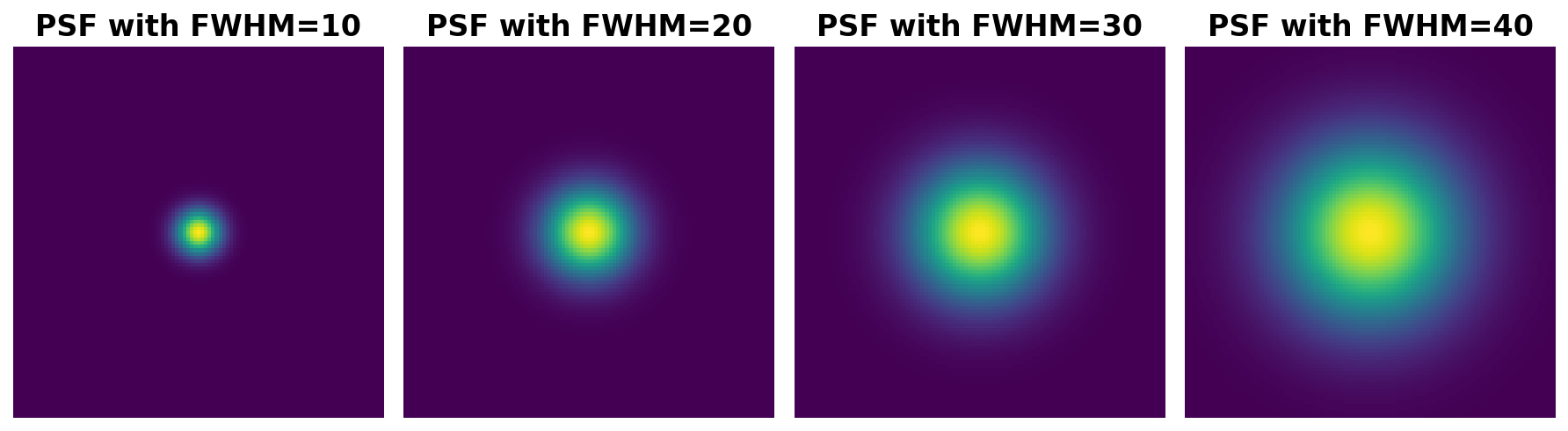}
        \label{fig:psf_small}
    
    \vspace{0.5em} 
    
        \centering
        \includegraphics[width=\linewidth]{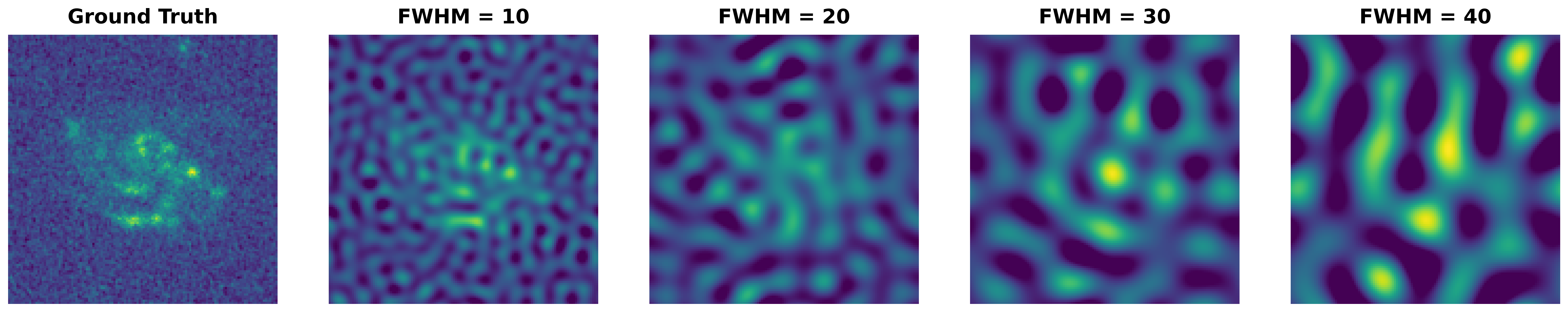}
        \label{fig:psf_large}
    \label{fig:psf_increasingfwhm}
    \caption{Point spread functions (PSFs) with varying FWHM values and the corresponding ground-truth image after convolution with each PSF.}
\end{figure}

In this experiment, we perturb the inputs by varying the FWHM of the PSF, and compute the corresponding CHEM for the wavelet coefficients of the image as described in Algorithm~\ref{alg:CHEM}. Figure~\ref{fig:FWHM-analysis} illustrates that both SUNet and U-Net achieve a lower performance when employing $\ell_1$ loss, and they also exhibit a decrease in the extent of hallucinations with minor perturbations. This suggests that in large neural networks, the application of loss $\ell_1$ could be advantageous to enhance performance while minimizing hallucination levels. On the other hand, Learnlets achieves MSE similar to those of U-Net under both losses, but its effectiveness in CHEM is significantly better than that of U-Net. Therefore, CHEM might provide an alternative viewpoint on the behavior of the model beyond just measuring the MSE.
\begin{figure}[h!]
        \centering
        \includegraphics[width=\linewidth, keepaspectratio]{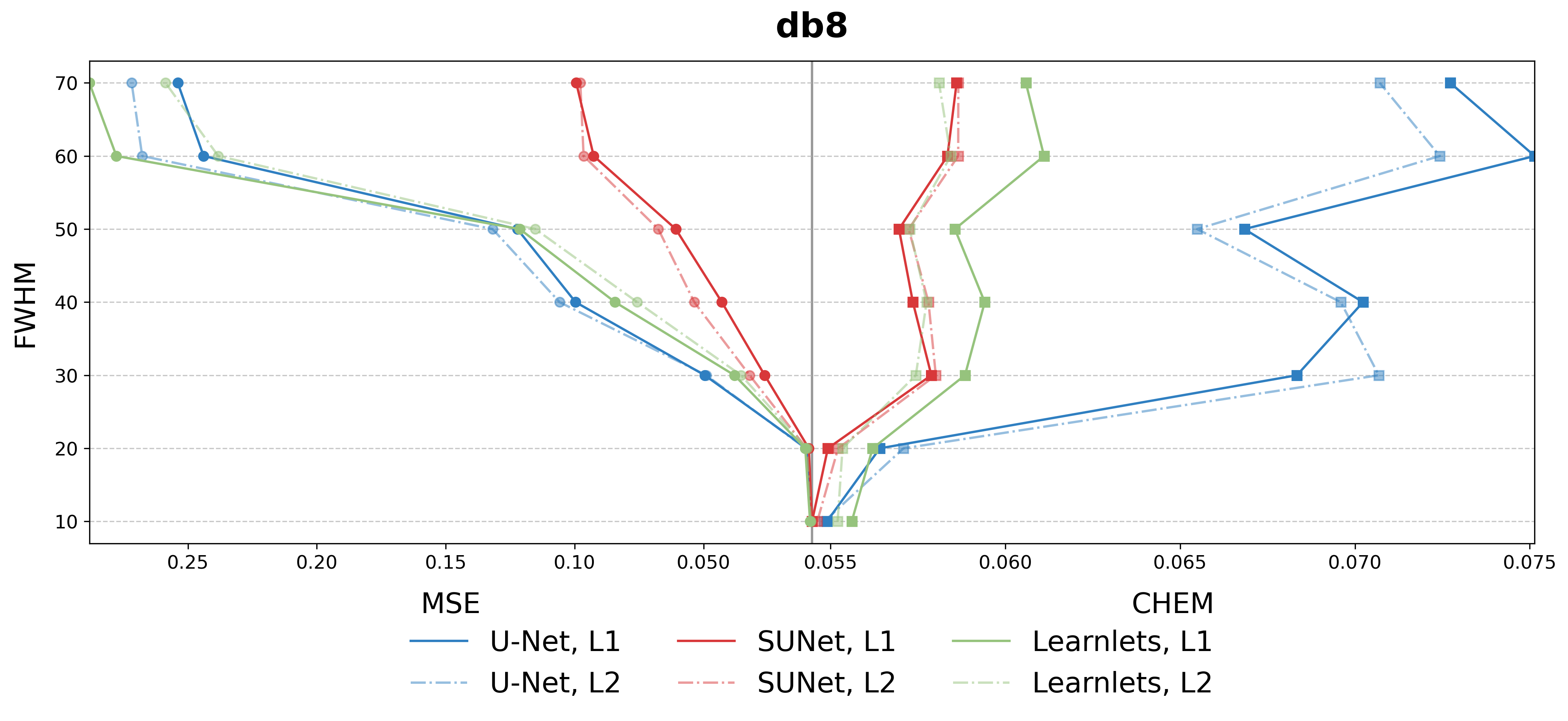}
    \caption{
        MSE/CHEM-FWHM curves. Analysis of the changes in MSE/CHEM with varying FWHM of disturbed inputs to test the robustness of U-shaped architectures.
    }
    \label{fig:FWHM-analysis}
\end{figure}

\subsubsection{Effect of dictionary choice on hallucination detection}

\label{Subsec: Hallucination under Varying Dictionaries}
To further analyze the impact of the chosen representation, we additionally considered Daubechies-4 (db4) and shearlets. Figure~\ref{fig:Varying Dictionaries: Pyramids} summarizes the results for these additional dictionaries. Although some variations can be observed when changing the transformation, the general trend remains consistent with that of Figure~\ref{fig:FWHM-analysis}: Learnlets- and SUNet-based deconvolution methods exhibit greater stability than U-Net under CHEM.
\begin{figure*}[h!]
    \centering
    \includegraphics[width=0.9\linewidth]{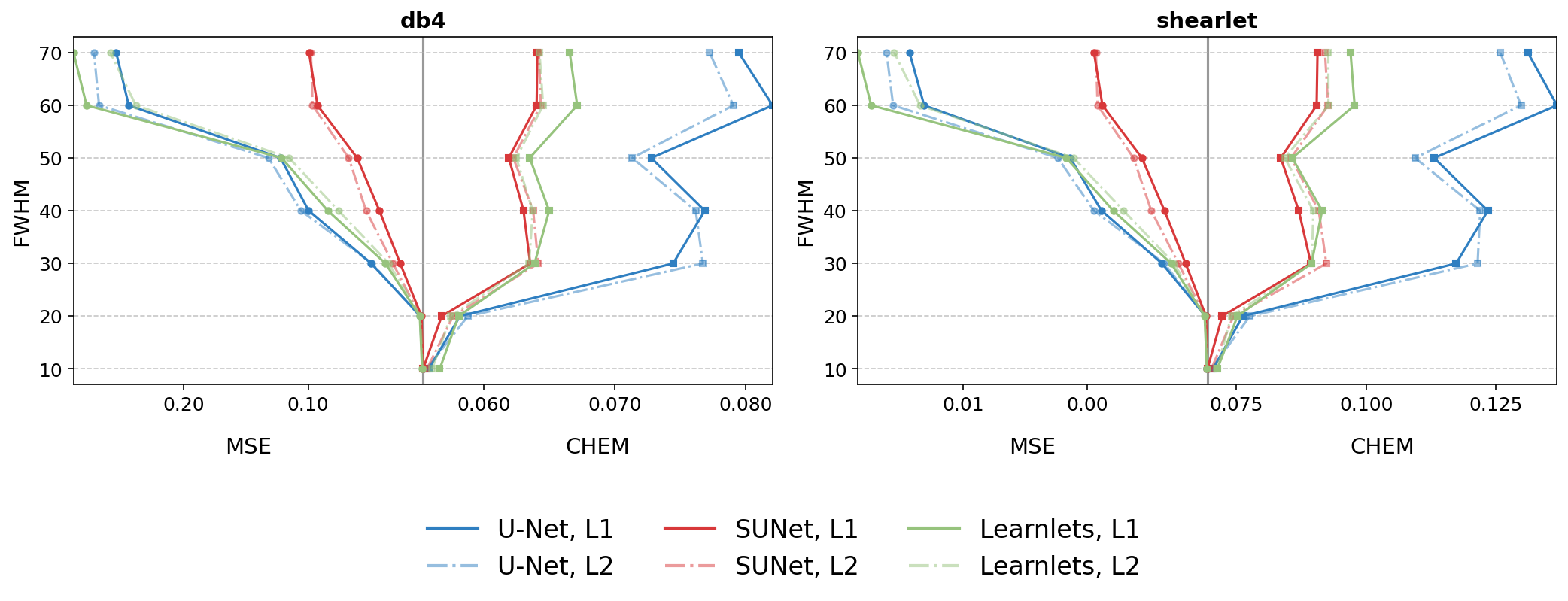}

    \caption{
        MSE/CHEM-FWHM curves under different dictionaries. 
        This figure illustrates the effect of the chosen representation.
    }
    \label{fig:Varying Dictionaries: Pyramids}
\end{figure*}
Visually, the effect of the chosen dictionary is illustrated in Figure~\ref{fig:cross-dict-hallucination}. The reconstructed maps obtained using db8 and shearlets provide a clearer representation of the hallucinations present in the prediction. This is because these representations are particularly well suited to capturing the structures characteristic of this application.
\begin{figure*}[h!]
    \centering
    \includegraphics[width=\linewidth]{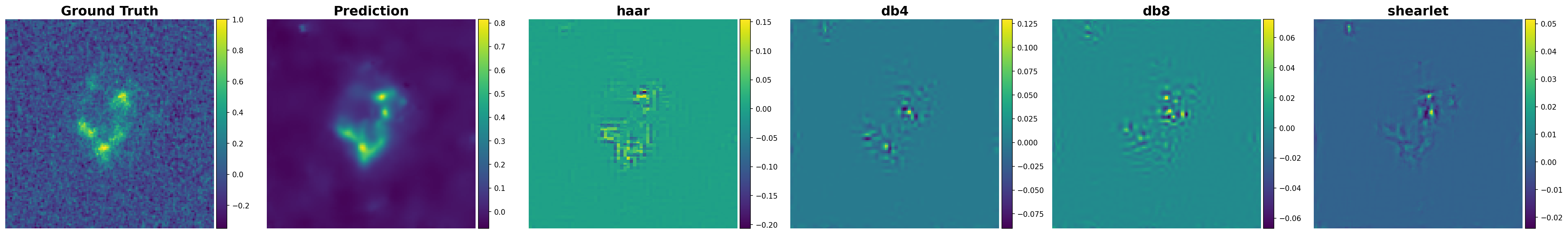}
    \caption{Quantifying hallucinations of a U-Net trained with $\ell_2$ loss on astronomical images using different dictionaries. Visually, db8 and shearlets provide a clearer representation of the hallucinations present in the prediction. The reconstructions correspond to the fine-scale coefficients of these dictionaries with high CHEM values. For details, please refer to Section \ref{sec: ExperimentalSetupAstro}. }
    \label{fig:cross-dict-hallucination}
\end{figure*}
\subsubsection{Hallucination-distortion tradeoff}
\label{subsec: Hallucination-Performance Tradeoff}
We now compare hallucination with training loss to examine CHEM during training. Figure~\ref{fig:MSE vs Hallucination Index} shows the evolution of hallucination with respect to the corresponding training loss over the first 300 epochs. The metrics are normalized per model using a min–max scaling. The tradeoff between MSE and CHEM is most evident for U-Net: beyond a certain number of epochs, further reductions in training loss come at the cost of increased hallucination. The tradeoff phenomenon of SUNet is also observed after about 200 epochs. In contrast, Learnlets exhibits the most stable behavior, maintaining low hallucination levels throughout training. It may, however, approach a tradeoff beyond 300 epochs.
\begin{figure}[h!]
    \centering
        \includegraphics[width=\linewidth, keepaspectratio]{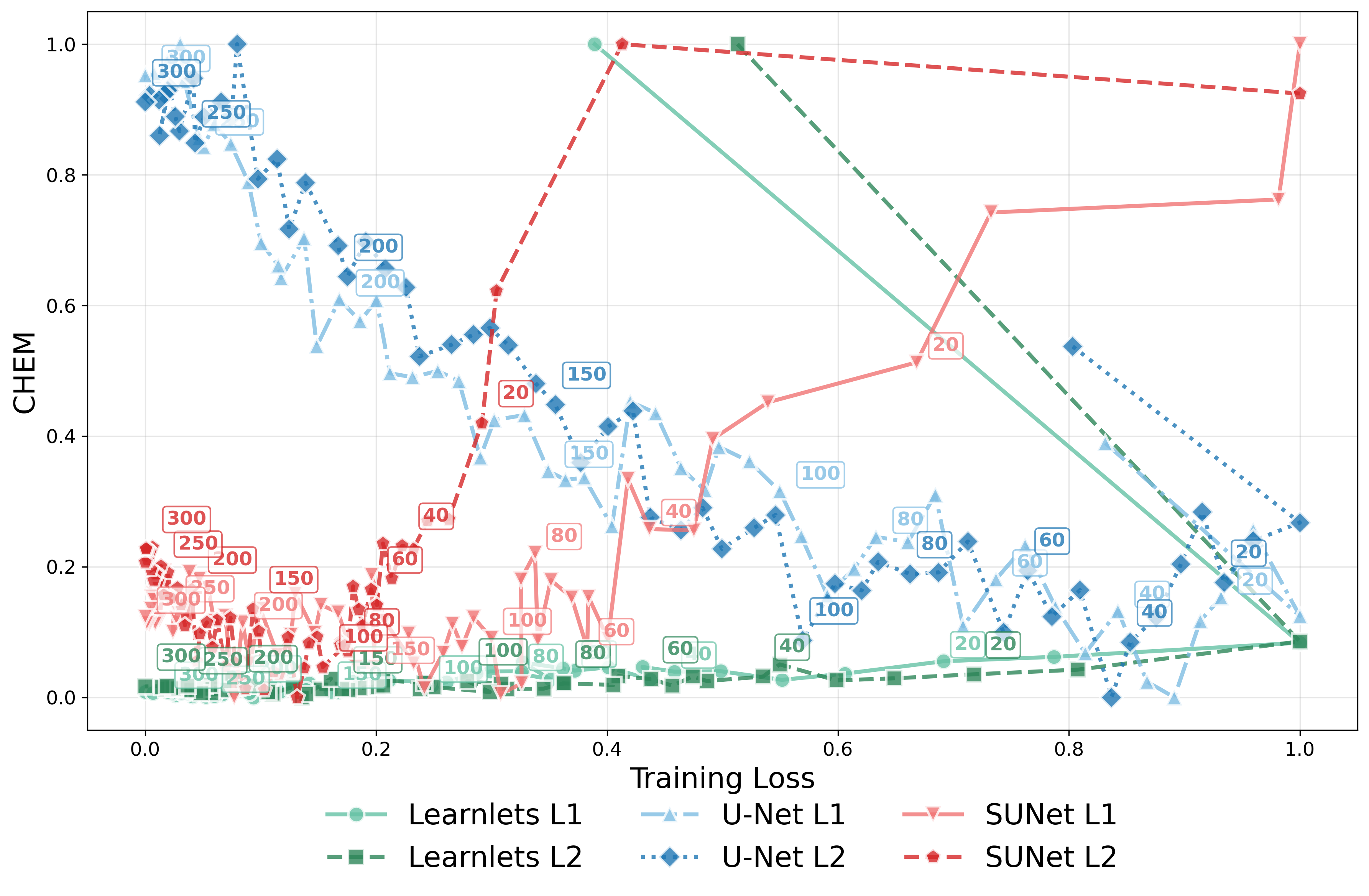}
    \caption{Evolution of the db8-based CHEM and training loss over the first 300 epochs, averaged across all test images. All metrics are normalized per model.}
    \label{fig:MSE vs Hallucination Index}
\end{figure}
These results align with the results reported in \cite{cohen_looks_2024}.

\subsection{Image superresolution on natural images}
We extend the evaluation to natural image super-resolution, using pre-trained models on the DIV2K dataset to test the generalizability of CHEM beyond the astronomical domain.
\subsubsection{Experimental setup}
\label{sec: N2 Experimental Setup}
\paragraph{Dataset}
For the super-resolution experiments we use the DIV2K dataset \cite{8014884}, a widely used benchmark for single-image super-resolution. The dataset consists of 1,000 high-resolution natural images (2K resolution), divided into 800 training images, 100 validation images, and 100 test images. Since the models we use in this setting are pre-trained, the dataset is only used for calibration and validation. Specifically, we extract 200 patches of size $256 \times 256$ from 50 images for calibration, and 50 patches of the same size from a disjoint set of 50 images for validation.

\paragraph{Image super-resolution}
The super-resolution problem consists of recovering a high-
resolution image \(x\) from a low-resolution observation \(y\). We model the forward process as
\begin{equation}
y = SBx ,
\end{equation}
where \(B\) denotes a blurring operator (convolution with a blur kernel) and \(S\) is a subsampling operator representing spatial downsampling. In our experiments we consider a scale factor of \(4\), corresponding to \(4\times\) spatial upsampling.

\paragraph{Models}
For the super-resolution experiments we evaluate several pretrained reconstruction models available in the \texttt{deepinv} library \cite{tachella2025deepinverse}. 
Specifically, we consider bicubic upsampling as a baseline, DRUNet-Pnp \cite{zhang_plug-and-play_2021}, Unfolded DRS (Douglas–Rachford Splitting) \cite{Douglas1956OnTN}, RAM (Reconstruct Anything Model) \cite{terris2025reconstruct}, and DPS (Diffusion Posterior Sampling) \cite{chung2022diffusion}.\\
Bicubic upsampling is a classical interpolation-based method used as a baseline in super-resolution. The high-resolution image is obtained by applying cubic convolution in both spatial dimensions, where each pixel is computed from a weighted \(4 \times 4\) neighborhood of the low-resolution image.
Plug-and-play (PnP) restoration methods formulate an inverse problem as the minimization of an objective functional consisting of a data-fidelity term $f(x,y)$ and a regularization term $\mathcal{R}(x)$ for input $y$ and ground truth $x$:
$\underset{x}{\text{argmin}} f(x,y) + \lambda \mathcal{R}(x).$ In the DRUNet-based PnP \cite{zhang2021plug} framework, model-based optimization is combined with learning-based priors. The reconstruction problem is addressed via a variable-splitting scheme, which decouples the data-fidelity term from the prior term. The regularization step is replaced by a deep convolutional denoiser. 
Specifically, DRUNet serves as the denoising operator within the iterative scheme. Architecturally, DRUNet combines U-Net–type multiscale feature extraction with residual connections.
An alternative strategy for integrating model-based optimization and learned priors is deep unfolding. Here, an iterative optimization algorithm, such as Douglas–Rachford splitting, is interpreted as a finite-depth network. The classical Douglas–Rachford scheme solves composite minimization problems by alternating proximal steps associated with the data-fidelity and regularization terms. Algorithm unrolling \cite{monga_algorithm_2020} interprets an iterative procedure that repeatedly applies an analytic operation \(h\) as a neural network, where each iteration is mapped to a single network layer and a finite number of such layers are stacked to form the architecture.
The Reconstruct Anything Model (RAM) \cite{terris2025reconstruct} is a foundation model for inverse problems built upon the DRUNet convolutional neural network architecture. Its key modification consists of a conditioning mechanism that incorporates the forward measurement operator \(A\) and the observations \(y\) via multigrid Krylov iterations. This design enables a single backbone model to be trained jointly across multiple imaging tasks. 
Diffusion Posterior Sampling (DPS) \cite{chung2022diffusion} generalizes diffusion-based generative models to the setting of noisy and potentially nonlinear inverse problems, by approximating the posterior sampling.
\begin{table}[ht]
\centering
\scriptsize
\caption{Inference time  of the evaluated reconstruction methods for single image.}
\begin{tabular}{lc}
\toprule
Method & Time (s) \\
\midrule
Bicubic        & 0.545 \\
DRUNet-PnP     & 8.494 \\
Unfolded-DRS   & 0.834 \\
RAM            & 6.073 \\
DPS            & 82.040 \\
\bottomrule
\end{tabular}
\label{tab:runtime}
\end{table}

\paragraph{Evaluation Metrics and CHEM computation}
We follow the same evaluation protocol established in
Section~\ref{sec:EvMetrics}: reconstruction accuracy is measured by
MSE, hallucination is quantified by CHEM (Algorithm~\ref{alg:CHEM},
$\alpha=0.01$, $\theta=1$).
Hallucination maps are produced using
the same per-coefficient standardization procedure. The threshold is set to $1.5$.
Again, restricting to the two
finest scales empirically yields the sharpest localization of
hallucinations.
 In Figure~\ref{fig: SR decomposition}, we apply a db8 wavelet decomposition to the image and show the corresponding MSE and hallucination maps across different scales to further illustrate this effect. 
To better visualize regions with high hallucination risk  in natural images, Figures~\ref{fig:CHEM Superresolution visualization} follows a similar procedure, now highlighting regions corresponding to coefficients exceeding a threshold of $1.5$, overlaid on the predicted image.

\subsubsection{Model comparison}

\begin{figure*}
    \centering
    \includegraphics[width=\linewidth]{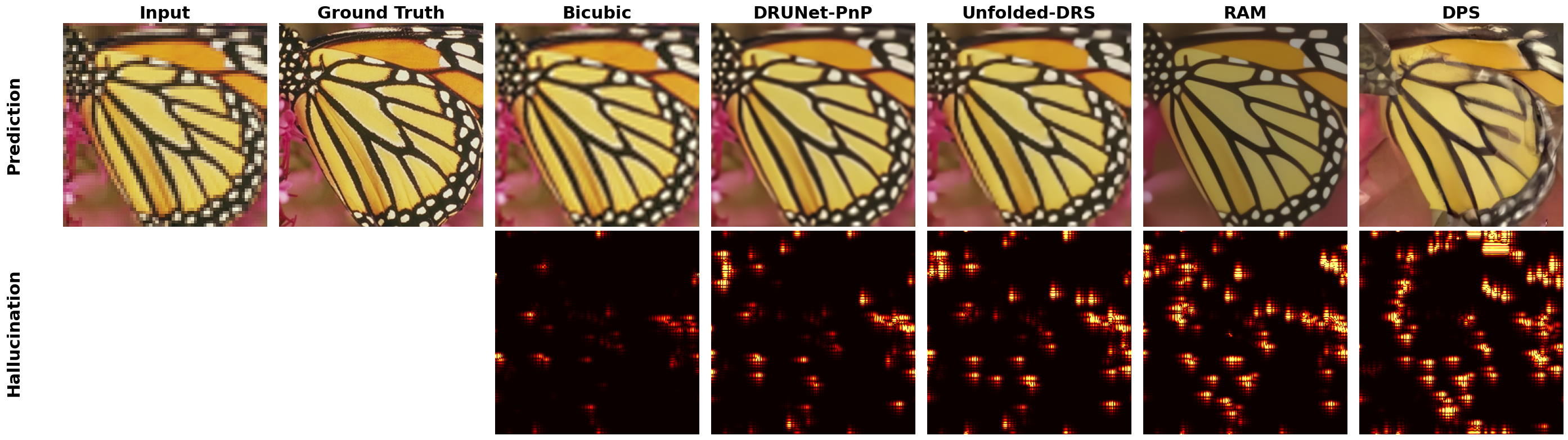}
    \caption{Quantifying hallucinations in $4\times$ super-resolution reconstructions produced by different models using db8 wavelets. The reconstructed images are shown in the first row, while the second row displays the corresponding CHEM values for fine-scale coefficients $\operatorname{H}^{\theta}(\Phi)_j$ exceeding a threshold. In particular, the DPS reconstruction exhibits substantially more high-CHEM regions than the classical bicubic baseline. For further details on the evaluation metrics and CHEM computation, see Section~\ref{sec: N2 Experimental Setup}.}
    \label{fig:N2 Predictions}
\end{figure*}
Figure~\ref{fig:N2 Predictions} illustrates regions with high CHEM values across different super-resolution models. The first row shows the reconstructed images, while the second row visualizes CHEM values corresponding to fine-scale db8 wavelet coefficients that exceed a threshold. The colorbar is fixed across models.

The DPS reconstruction exhibits more regions where fine-scale coefficients exceed the selected threshold. This observation suggests that the use of a threshold plays an important role when CHEM is employed for spatial localization of hallucinated structures, as opposed to dataset-level averages that capture the overall hallucinatory behaviour of a model.


\subsubsection{Effect of dictionary choice on hallucination detection}

\begin{figure}
    \centering
    \includegraphics[width=\linewidth]{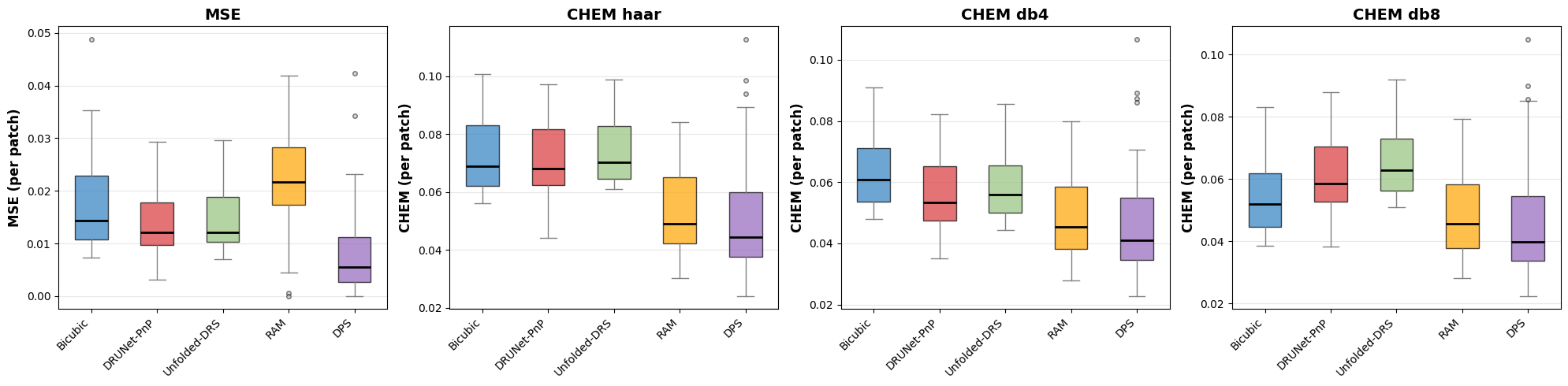}
    \caption{Comparison of reconstruction performance and hallucination levels across different super-resolution models under $4\times$ upsampling. The first panel shows the mean squared error (MSE), while the remaining panels report CHEM values computed using Haar, db4, and db8 wavelet dictionaries. For each image, the metric is first averaged over all pixels (for MSE) or transform coefficients (for CHEM) to obtain an image-level score. The boxplots summarize the distribution of these scores across the validation dataset.}
    \label{fig:boxplots}
\end{figure}

The first panel of Figure~\ref{fig:boxplots} reports the MSE distribution across models, while the remaining panels show the corresponding CHEM values computed using different dictionaries (Haar, db4, and db8). Although the average MSE per patch may suggest superior performance of DPS, an examination of the corresponding CHEM values, particularly the outliers, indicates weaker performance of this model with respect to hallucination. Conversely, when assessed solely by MSE, RAM appears to perform worse than competing methods. However, evaluation via CHEM places this model among the strongest performers, as it exhibits lower CHEM values. The absolute CHEM values vary slightly across dictionaries, however the overall trends remain consistent: models that exhibit stronger hallucination behavior under one dictionary tend to do so across the others as well. This suggests that while the choice of dictionary influences the sensitivity of CHEM to specific image structures, the qualitative comparison between models remains stable. In practice, the dictionary should therefore be chosen to match the structural characteristics of the images under consideration.

These results emphasize that MSE alone does not fully characterize hallucination behavior.

\subsubsection{Visual hallucination analysis}
Finally, we present an exemplary hallucination map. The map is constructed as mentioned in Section \ref{sec: N2 Experimental Setup}. In Figure \ref{fig:CHEM Superresolution visualization} (a), the first row displays the full image, while the second row shows a zoomed-in view of the top-right corner.
The MSE map captures all reconstruction errors without distinction. In contrast, CHEM selectively highlights regions with high hallucination risk. \\
In Rectangle $1$, the model generates a blob and a vertical, rounded structure that are not present in the ground truth. Moreover, the lower boundary is reconstructed with rounded contours rather than the straight edge observed in the ground truth. These regions are emphasized by CHEM. The sensitivity to such rounded structures is consistent with the use of the db8 wavelet dictionary for visualization.\\
In Rectangle $2$, the prediction alters the geometry of the black line, which becomes straight downward instead of first slightly upward and then downward as in the ground truth. Additionally, a white triangular texture is replaced by a smooth transition from a lighter to a darker shade. These structural modifications are likewise captured by CHEM.\\
In Rectangle $3$, the prediction introduces a triangular structure consisting of an inner black region and an outer white boundary, which are not present in the ground truth. These artificial edge patterns are highlighted by CHEM, which assigns elevated values to the corresponding region. Figure~\ref{fig:CHEM Superresolution visualization} (b) provides an additional example.
\begin{figure*}
    \centering
    \begin{subfigure}{\linewidth}
        \centering
        \includegraphics[width=\linewidth]{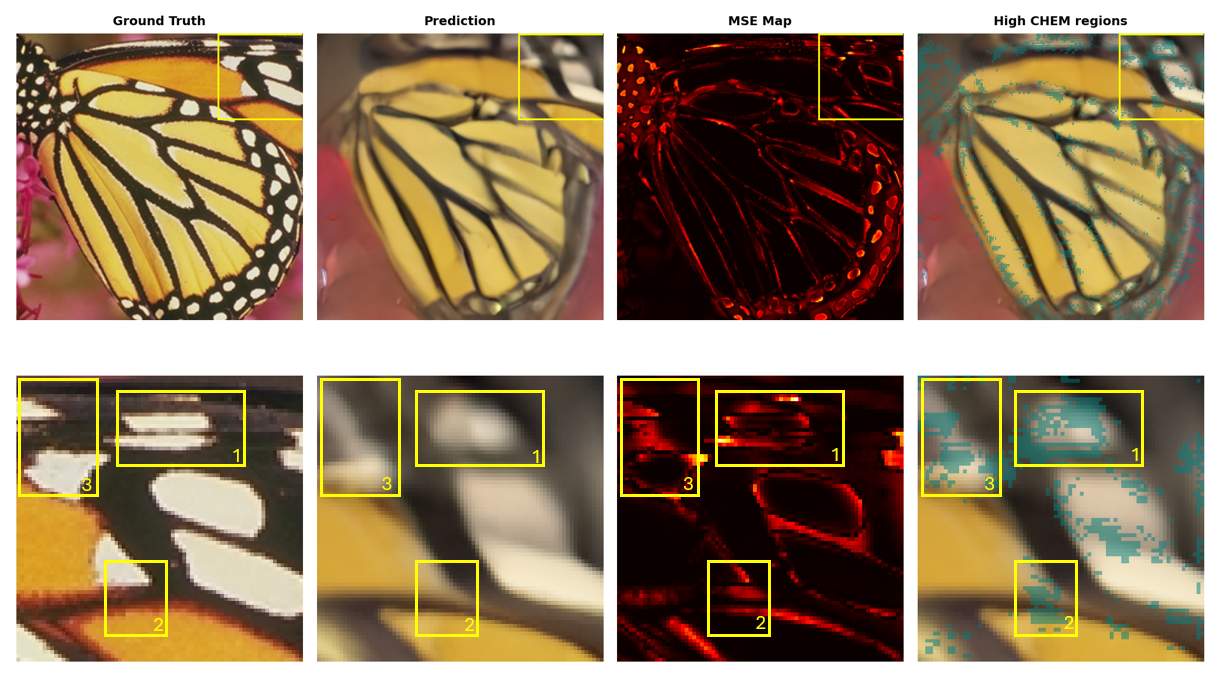}
        \caption{}
        \label{fig:CHEM SR viz A}
    \end{subfigure}
    
    \begin{subfigure}{0.98\linewidth}
        \centering
        \includegraphics[width=\linewidth]{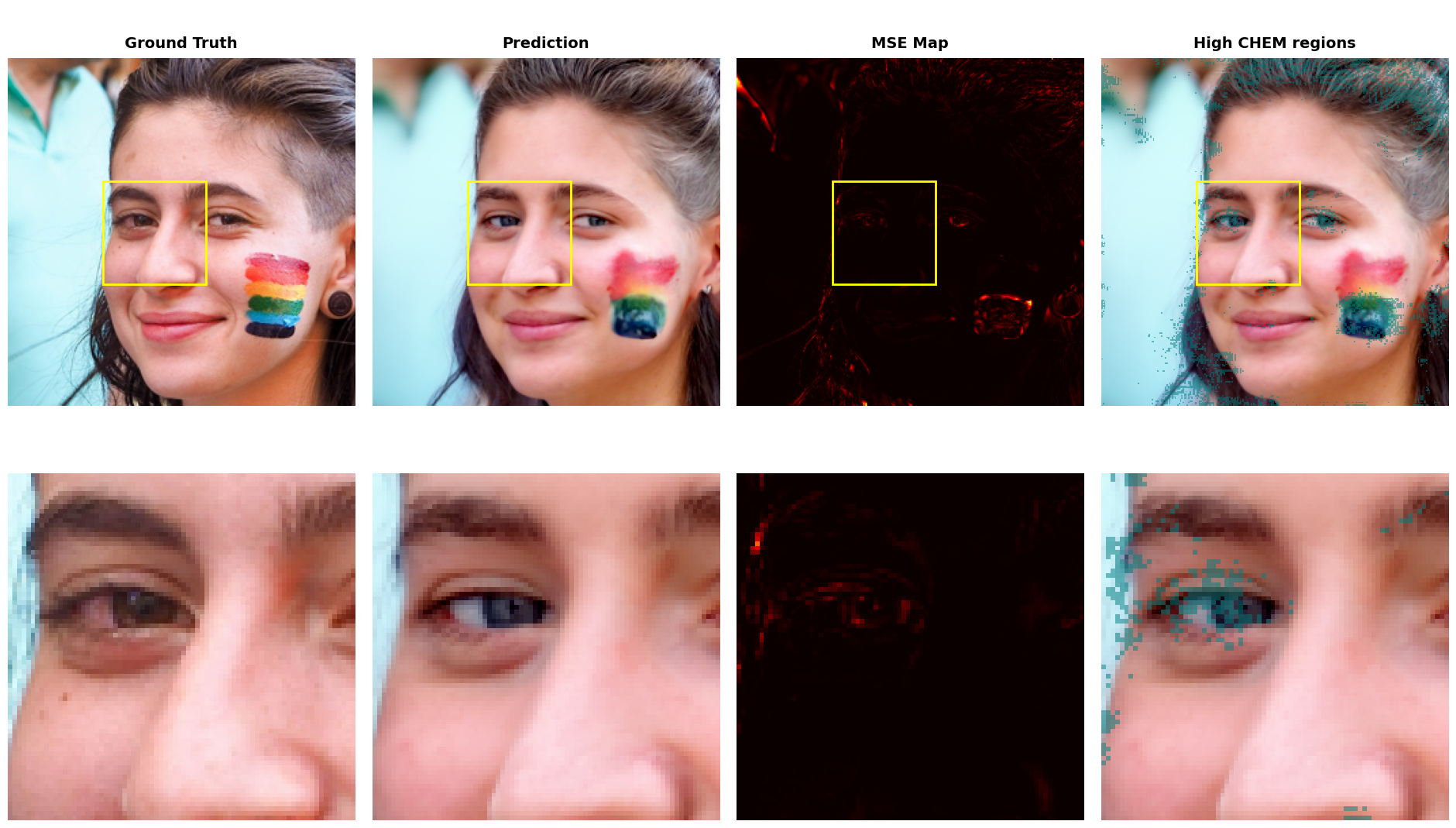}
        \caption{}
        \label{fig:CHEM SR viz B}
    \end{subfigure}
    
    \caption{Visualization of hallucination-prone regions in a $4\times$ 
    super-resolution reconstruction. From left to right: ground truth, 
    model prediction, MSE map, and regions where CHEM, computed using 
    db4 wavelets, exceeds a chosen threshold. The second row shows 
    zoomed-in views of the first row, with panel (b) illustrating that CHEM can identify localized deviations even when the corresponding MSE map has very low visual contrast. See Section~\ref{sec: N2 
    Experimental Setup} for details on the computation.}
    \label{fig:CHEM Superresolution visualization}
\end{figure*}

Figure~\ref{fig: SR decomposition} aims to further illustrate the relevance of CHEM at high-frequency levels of a decomposition. It shows the reconstructed image, the ground truth, and the corresponding total MSE and CHEM maps, followed by their db8 wavelet decompositions across scales. 
While both MSE and CHEM capture discrepancies between the reconstruction and the ground truth, CHEM becomes more informative at finer scales by more clearly localizing structural deviations introduced by the model, particularly in regions such as the painted cheek area, the ear, and the eye contours. Crucially, this illustrates that as the decomposition progresses toward finer levels, the measured hallucination becomes progressively more isolated and spatially concentrated, while MSE loses sensitivity. 

\begin{figure*}
    \centering
    \includegraphics[width=\linewidth]{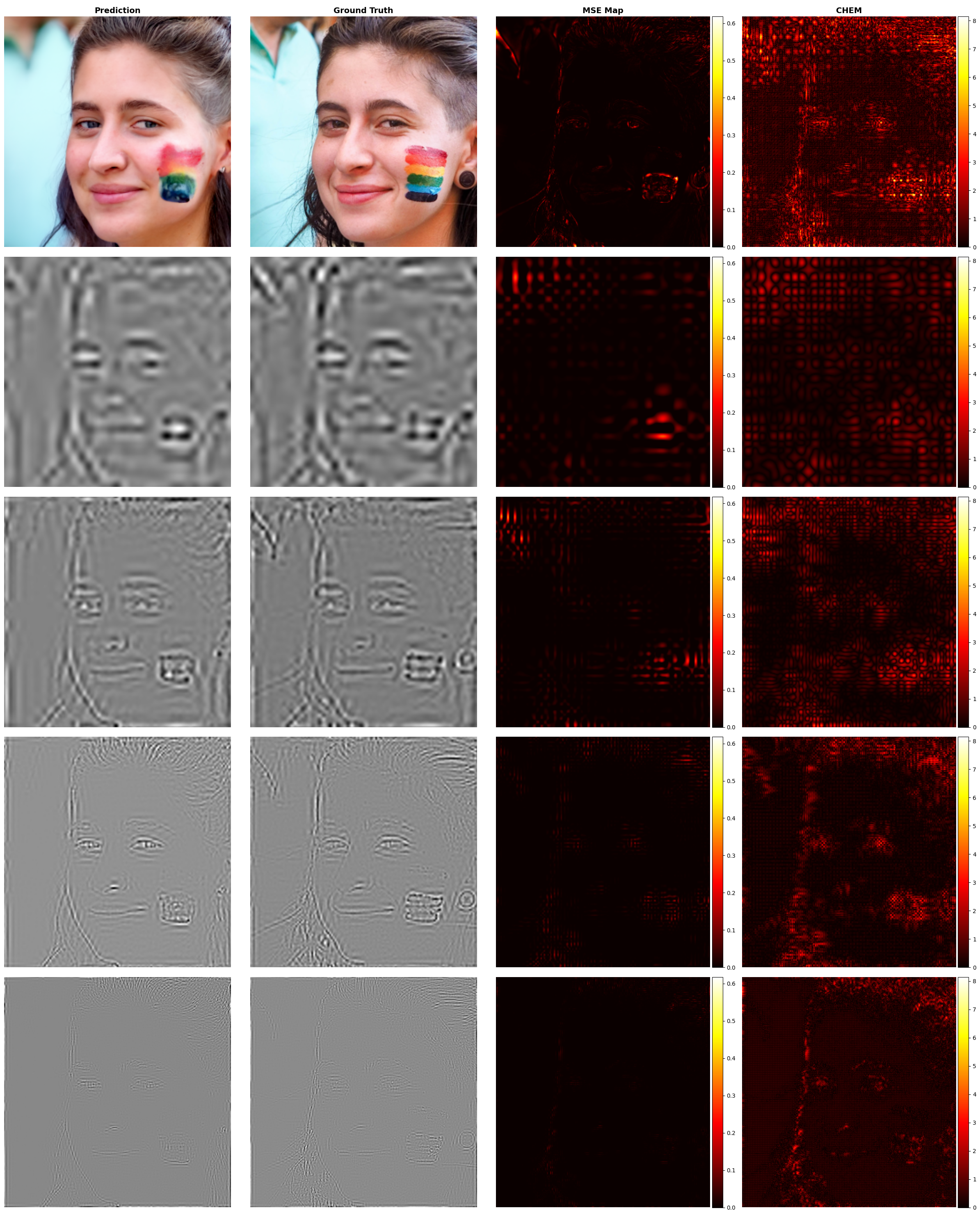}
    \caption{Hallucination map in a $4\times$ super-resolution reconstruction produced by DPS using a db8 wavelet decomposition. 
    From left to right: prediction, ground truth image, pixel-wise MSE map, and CHEM map. The rows below show the db8 wavelet decomposition of each quantity across successive levels, from coarse to fine scales. At finer scales, CHEM highlights localized structural deviations introduced by the model that are less clearly captured by the MSE, as reflected by the nearly empty fine-scale MSE maps.}
    \label{fig: SR decomposition}
\end{figure*}

\subsection{Comparison with existing approaches}
The metrics mentioned in Section \ref{sec: Related Work} differ in several aspects: the application they are tailored for \cite{akhaury_ground-based_2024,kolek_explaining_2023}, the distribution assumptions they rely on \cite{tivnan_hallucination_2024}
, or whether they provide a whole-model measure \cite{cohen_looks_2024}, a per-image score \cite{ren2026hallucinationscoremitigatinghallucinations}, or a spatial map of hallucination-prone regions \cite{bhadra_hallucinations_2021}. A fundamental challenge in the hallucination field is the absence of ground-truth hallucinations: our interpretation of hallucination depends on human perception and domain knowledge. A fundamental challenge in the hallucination field is the absence of ground-truth hallucinations: our interpretation depends on human perception and domain knowledge.
Accordingly, within the body of existing literature, we benchmark our method against: (i) the widely recognized null-space hallucination map proposed in \cite{bhadra_hallucinations_2021}, and (ii) the uncertainty heatma introduced by \cite{angelopoulos2022image}. We assess these metrics by comparing the hallucinations they identify, which allows us to situate our metric within existing work and to gauge how effective these metrics are at detecting localized artifacts.

\subsubsection{Comparison metrics}
\paragraph{Null-space hallucination map (Null-HM)}
Bhadra et al.\ \cite{bhadra_hallucinations_2021} propose to decompose a
reconstructed image into generalized measurement-space and null-space
components via the singular value decomposition (SVD) of the forward
operator.  The null-space hallucination map captures reconstruction errors
attributable to the imposed prior by comparing the null-space component of
the prediction against that of the ground truth.  We note that the
computation requires either an explicit SVD of the forward operator or
direct access to its (pseudo)inverse.  In our applications either one is
available, but this requirement may limit applicability to more general
reconstruction frameworks.

\paragraph{Risk-controlling prediction sets (RCPS)}
Angelopoulos et al.~\cite{angelopoulos2022image} construct pixelwise
confidence intervals with a distribution-free statistical guarantee: given a user-specified risk level~$\alpha$ and confidence level~$1-\delta$, the
intervals contain at least a fraction~$1-\alpha$ of the ground-truth pixel
values with probability at least~$1-\delta$.  The method proceeds in two
stages: a heuristic interval is first produced by a network trained to
predict both predictions and interval widths, and is then calibrated
into risk-controlling prediction sets by scaling with a factor
$\hat{\lambda}$ determined via the Hoeffding--Bentkus upper confidence bound (UCB)~\cite[Algorithm~2]{angelopoulos2022image}.
The original formulation trains the heuristic interval jointly with the
reconstruction model, which is impractical when comparing across multiple
pretrained models.  We therefore adopt the residual-magnitude heuristic
of~\cite[Section~2.1.1]{angelopoulos2022image} in a post-hoc fashion: a
separate lightweight U-Net is trained to
predict the pixelwise absolute residual $|f(x) - y|$ from the
reconstruction model's output, using the first half of the calibration data
described in Section~\ref{sec: ExperimentalSetupAstro}.  The second half is
used to calibrate the resulting heuristic intervals via the
Hoeffding--Bentkus UCB, following~\cite[Algorithm~2]{angelopoulos2022image}.

\paragraph{Remarks}
We highlight that, CHEM is computationally more efficient than both alternatives, requiring neither the SVD decomposition of \cite{bhadra_hallucinations_2021} nor the residual prediction network of \cite{angelopoulos2022image}, making it particularly well-suited for high-dimensional tasks. We also note that, unlike RCPS which operates at the pixel level, CHEM operates in the wavelet or shearlet domain, and may therefore concentrate its detection on spatially coherent, localized artifacts rather than distributing scores uniformly across the image. 
In what follows, we focus on the quality of hallucination detection. 

\subsubsection{Spatial comparison measures}

We quantify the spatial relationship between hallucination maps using two
measures, described in the following.


\paragraph{Thresholded intersection-over-union (IoU)}
For a percentile threshold $q \in (0,100)$, let \newline
$B_q(\phi) = \{i :|\phi_i| \geq \mathrm{percentile}_q(|\phi|)\}$ denote the set of pixels
exceeding the $q$-th percentile of absolute values.  The IoU between two
maps at threshold $q$ is
\begin{equation*}
  \mathrm{IoU}_q(\phi,\psi)
  \;=\;
  \frac{|B_q(\phi) \cap B_q(\psi)|}{|B_q(\phi) \cup B_q(\psi)|}\,.
\end{equation*}
The chance-level
IoU, under the assumption of drawing uniformly at random from $\{1,\dots,N\}$ with no spatial
dependence, is given by 
\begin{equation*}
  \mathrm{IoU}_q^{\mathrm{ch}}
  \;=\; \frac{100-q}{100+q}.
\end{equation*}
 To facilitate
interpretation, we report the ratio
$\mathrm{IoU}_q\!/\,\mathrm{IoU}_q^{\mathrm{ch}}$ alongside each
measured IoU value; a ratio of~$1$ indicates agreement no better than
independent random flagging.

\paragraph{Entropy-based effective extent}
Following~\cite{kolek_explaining_2023}, we quantify the spatial concentration
of each map as follows.  Given a
spatial map $\phi$ with pixel values $\{\phi_i\}_{i=1}^N$, we define the
energy distribution $p_i = |\phi_i|^2 \big/ \sum_{j=1}^N |\phi_j|^2$ and
compute the exponential of the entropy $\exp\!\Bigl(-\sum_{i=1}^{N} p_i \ln p_i\Bigr)$, called the \textit{extent} \cite{campbell_exponential_1966}.
The extent can be interpreted as the effective number of pixels over
which the map distributes its energy: a map concentrating all energy on a
single pixel has $\mathrm{Extent} = 1$, while a spatially uniform map
attains $\mathrm{Extent} = N$.  The \emph{relative extent},
$\mathrm{Extent}(\phi)/N$, provides a measure of
the fraction of the image effectively covered by the map.
\begin{figure*}[h!]
    \centering
    \includegraphics[width=\linewidth]{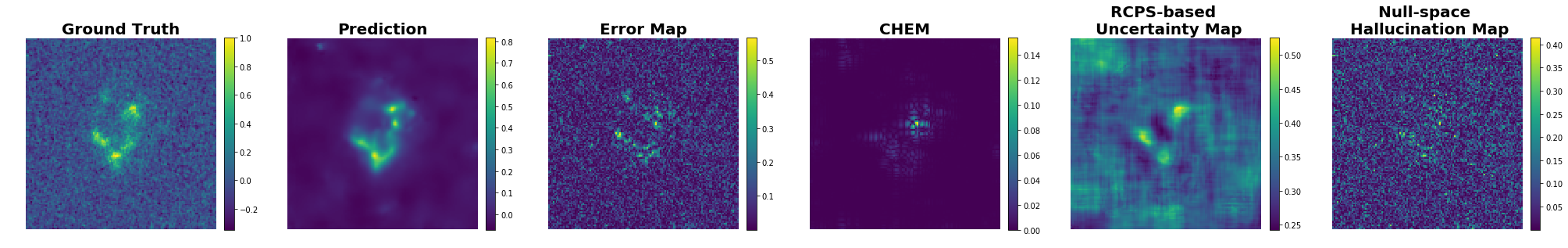}
    \caption{Representative comparison of spatial maps for the astronomical
    deconvolution setting. CHEM maps exhibit visibly higher spatial
    concentration than both RCPS and Null-HM.}
    \label{fig:A1 Comparison}
\end{figure*}
\begin{table*}[h!]
\centering

\begin{minipage}{0.68\textwidth}
\centering
\caption{Thresholded intersection-over-union between top-flagged pixels for the
astronomical deconvolution setting. Each cell reports
$\mathrm{IoU}_q\;(r\!\times)$, where $r :=
\mathrm{IoU}_q/\mathrm{IoU}_q^{\mathrm{ch}}$.}
\label{tab:iou}
\small
\resizebox{\linewidth}{!}{%
\begin{tabular}{llccc}
\toprule
& & \textbf{90th} & \textbf{95th} & \textbf{99th} \\
& & {\scriptsize($\mathrm{IoU}^{\mathrm{ch}} = 0.053$)}
  & {\scriptsize($\mathrm{IoU}^{\mathrm{ch}} = 0.026$)}
  & {\scriptsize($\mathrm{IoU}^{\mathrm{ch}} = 0.005$)} \\
\midrule
\multicolumn{2}{c}{Null-HM vs. RCPS}
  & 0.07\;(1.4$\times$) & 0.05\;(2.1$\times$) & 0.04\;(7.6$\times$) \\
\midrule
\multirow{3}{*}{RCPS vs.}
& \ CHEM (Haar)
  & 0.11\;(2.1$\times$) & 0.12\;(4.6$\times$) & 0.11\;(21.3$\times$) \\
& \ CHEM (db4)
  & 0.17\;(3.2$\times$) & 0.18\;(7.1$\times$) & 0.13\;(26.9$\times$) \\
& \ CHEM (db8)
  & 0.15\;(2.9$\times$) & 0.13\;(5.2$\times$) & 0.06\;(12.2$\times$) \\
\midrule
\multirow{3}{*}{Null-HM vs.} 
& \ CHEM (Haar)
  & 0.10\;(1.9$\times$) & 0.07\;(2.8$\times$) & 0.05\;(10.8$\times$) \\
& \ CHEM (db4)
  & 0.09\;(1.7$\times$) & 0.07\;(2.7$\times$) & 0.05\;(10.9$\times$) \\
& \ CHEM (db8)
  & 0.09\;(1.7$\times$) & 0.06\;(2.4$\times$) & 0.03\;(5.6$\times$) \\
\bottomrule
\end{tabular}%
}
\end{minipage}
\hfill
\begin{minipage}{0.28\textwidth}
\centering
\caption{Entropy-based relative extent on a single image --- U-Net L2, image deconvolution.}
\label{tab:extent}
\small
\begin{tabular}{lc}
\toprule
\textbf{Map} &  \textbf{Relative extent} \\
\midrule
RCPS                & $0.977$ \\
Null-HM             & $0.465$ \\
$|\text{Error}|$    & $0.431$ \\
\midrule
CHEM (Haar)         & $0.024$ \\
CHEM (db4)          & $0.021$ \\
CHEM (db8)          & $0.015$ \\
CHEM (db12)         & $0.027$ \\
\bottomrule
\end{tabular}
\end{minipage}

\end{table*}

\begin{figure}[h!]
\centering
  \centering
  \includegraphics[width=\linewidth]{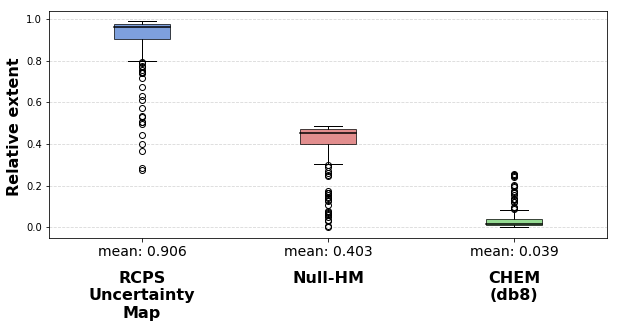}
  \caption{Distribution of entropy-based relative extent across the held-out test set for the U-Net astronomical image deconvolution setting.}
  \label{fig:A1_Comparison_Boxplots}
\end{figure}
\subsubsection{Results}
We focus on one model per application in order to compare the behavior of
the different spatial metrics rather than benchmark model performance.  
For astronomical image deconvolution, we consider a U-Net trained with $\ell_2$ loss.
Figure~\ref{fig:A1 Comparison} displays an example of the
three spatial maps alongside the reconstruction and ground truth for the
astronomical deconvolution setting. CHEM is strongly localized, and there is some agreement between the three maps. We explore these two observations further in Tables~\ref{tab:iou} and~\ref{tab:extent}.

Table~\ref{tab:iou} presents $\mathrm{IoU}_q$ and $\mathrm{IoU}_q^{\mathrm{ch}}$ for the example shown in Figure~\ref{fig:A1 Comparison}, computed between the different maps at the 90th, 95th, and 99th percentile thresholds. Two patterns emerge.  First, all Null-HM vs.\ CHEM
$\mathrm{IoU}_q$ values exceed $ \mathrm{IoU}_q^{\mathrm{ch}}$ at every percentile thresholds, with the $\mathrm{IoU}_q\!/\,\mathrm{IoU}_q^{\mathrm{ch}}$
ratio increasing from approximately $2$ at the 90th percentile to
$5$--$11$ at the 99th percentile.  This indicates that the two
hallucination metrics agree more strongly on the regions most prone to hallucinations
than on moderate ones, precisely the regime of greatest practical
concern. Second, RCPS exhibits higher overlap with CHEM than with Null-HM across all thresholds, suggesting that regions of very high uncertainty tend to coincide with hallucination-prone regions identified by CHEM.
\begin{figure*}[h!]
    \centering
    \includegraphics[width=\linewidth]{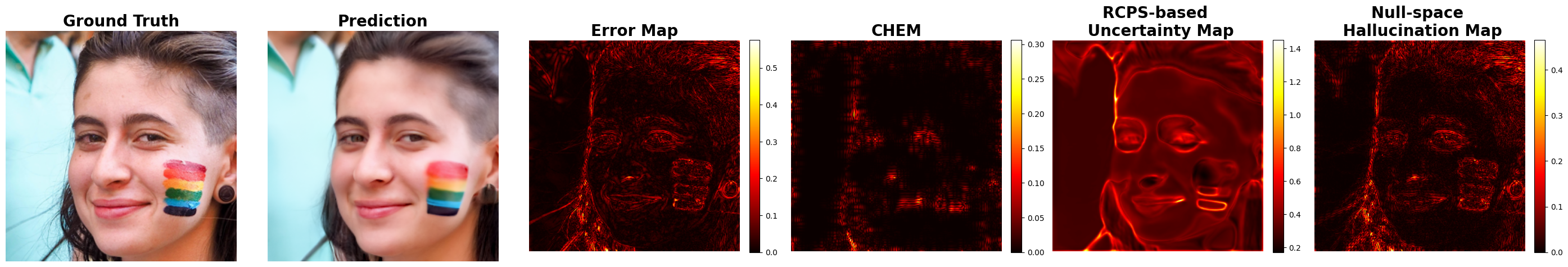}
    \caption{Representative comparison of spatial maps for the DRUNet-PnP
    $4\times$ super-resolution setting.}
    \label{fig:N2_comparison}
\end{figure*}

\begin{table*}[h!]
\centering

\begin{minipage}{0.68\textwidth}
\centering
\caption{Thresholded intersection-over-union between top-flagged pixels ---
DRUNet-PnP $4\times$ super-resolution. Each cell reports
$\mathrm{IoU}_q\;(r\!\times)$, where $r :=
\mathrm{IoU}_q/\mathrm{IoU}_q^{\mathrm{ch}}$.}
\label{tab:sr_iou}
\small
\resizebox{\linewidth}{!}{%
\begin{tabular}{llccc}
\toprule
& & \textbf{90th} & \textbf{95th} & \textbf{99th} \\
& & {\scriptsize($\mathrm{IoU}^{\mathrm{ch}} = 0.053$)}
  & {\scriptsize($\mathrm{IoU}^{\mathrm{ch}} = 0.026$)}
  & {\scriptsize($\mathrm{IoU}^{\mathrm{ch}} = 0.005$)} \\
\midrule
\multicolumn{2}{c}{Null-HM vs. RCPS}
  & 0.19\;(3.6$\times$) & 0.12\;(4.7$\times$) & 0.04\;(8.4$\times$) \\
\midrule
\multirow{3}{*}{RCPS vs.}
& \ CHEM (Haar)
  & 0.23\;(4.3$\times$) & 0.18\;(7.2$\times$) & 0.10\;(19.3$\times$) \\
& \ CHEM (db4)
  & 0.20\;(3.8$\times$) & 0.14\;(5.6$\times$) & 0.04\;(8.9$\times$) \\
& \ CHEM (db8)
  & 0.18\;(3.4$\times$) & 0.13\;(5.1$\times$) & 0.08\;(15.5$\times$) \\
\midrule
\multirow{3}{*}{Null-HM vs.} 
& \ CHEM (Haar)
  & 0.22\;(4.2$\times$) & 0.15\;(6.0$\times$) & 0.06\;(12.8$\times$) \\
& \ CHEM (db4)
  & 0.20\;(3.8$\times$) & 0.14\;(5.3$\times$) & 0.06\;(11.6$\times$) \\
& \ CHEM (db8)
  & 0.19\;(3.6$\times$) & 0.13\;(5.2$\times$) & 0.08\;(15.4$\times$) \\
\bottomrule
\end{tabular}%
}
\end{minipage}
\hfill
\begin{minipage}{0.28\textwidth}
\centering
\caption{Entropy-based relative extent on a single image --- DRUNet-PnP $4\times$ super-resolution.}
\label{tab:sr_extent}
\small
\begin{tabular}{lc}
\toprule
\textbf{Map} & \textbf{Relative extent} \\
\midrule
RCPS              & $0.854$ \\
$|\text{Error}|$  & $0.195$ \\
Null-HM           & $0.164$ \\
\midrule
CHEM (Haar)       & $0.057$ \\
CHEM (db4)        & $0.072$ \\
CHEM (db8)        & $0.079$ \\
\bottomrule
\end{tabular}
\end{minipage}

\end{table*}

\begin{figure}[h!]
  \centering
  \includegraphics[width=\linewidth]{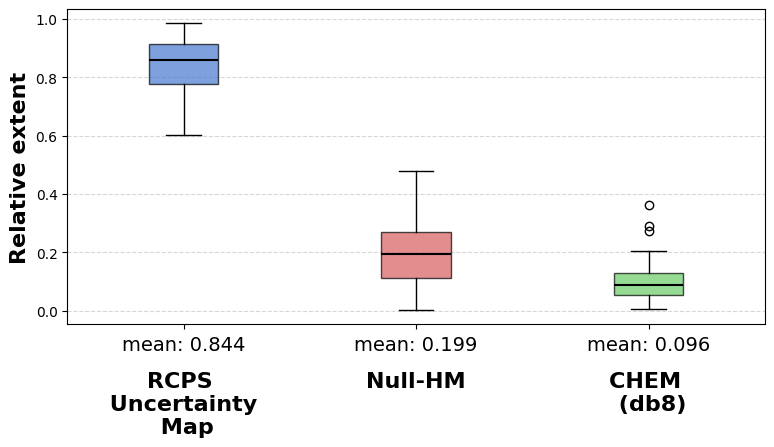}
  \caption{Distribution of entropy-based relative extent across the held-out test set for the DRUNet-PnP $4\times$ super-resolution setting.}
  \label{fig:N2_boxplots}
\end{figure}

Table~\ref{tab:extent} reports the entropy-based effective extent for all
maps for the example displayed in Figure \ref{fig:A1 Comparison}.  The RCPS uncertainty map exhibits a relative extent of~$0.98$, 
confirming that model uncertainty is distributed nearly uniformly across the
image without concentrating on specific regions.  The null-space
hallucination map and the absolute error map,  defined as the pixel-wise absolute difference between the reconstruction and the ground truth, show comparable extents ($0.46$
and~$0.43$, respectively), indicating moderate spatial concentration.
In contrast, CHEM maps achieve relative extents between~$0.015$
and~$0.027$, concentrating their energy on approximately $1.5$--$2.7\%$ of
the image. Figure \ref{fig:A1_Comparison_Boxplots} presents the distribution of relative extent across a held-out test set of $200$ images and confirms these observations.

We repeat the spatial comparison on a natural-image super-resolution task
using a  \newline DRUNet - PnP model with $4\times$ upsampling.

Table~\ref{tab:sr_iou} reports $\mathrm{IoU}_q$ and $\mathrm{IoU}_q^{\mathrm{ch}}$ at three percentile thresholds for the example shown in Figure~\ref{fig:N2_comparison}. A similar overall pattern in the astronomical context appears: all $\mathrm{IoU}_q$ scores are above $\mathrm{IoU}_q^{\mathrm{ch}}$, and the overlap between RCPS and CHEM remains higher than that between RCPS and Null-HM at most thresholds. This again indicates that areas of elevated uncertainty frequently align with regions that are susceptible to hallucinations.

Table~\ref{tab:sr_extent} reports the entropy-based effective extent for the example shown in Figure \ref{fig:N2_comparison}.  CHEM maps achieve relative extents of $0.057$--$0.079$, approximately $2$--$3\times$ larger than in the astronomical setting ($0.015$--$0.027$), but still $2$--$3\times$ more concentrated than Null-HM ($0.16$) and $11$--$15\times$ more concentrated than RCPS ($0.85$).
The ordering RCPS $\gg$ Null-HM $>$ CHEM is preserved across both settings. Figure \ref{fig:N2_boxplots}  validates these findings across the held-out test set of $200$ images, demonstrating consistent relative extent distributions.

The above results highlight the different design objectives of the three tools.
CHEM is constructed to localize individual hallucinated structures at
specific scales and orientations, yielding sparse, highly concentrated maps.
RCPS quantifies global posterior variability and is not intended to isolate
individual artifacts.  The null-space hallucination map occupies an intermediate
position: it targets prior-induced errors specifically, but its
pixel-domain formulation distributes energy more broadly than the CHEM maps.  Across both imaging settings, the spatial
concentration ordering and the pairwise overlap patterns between maps are
preserved.  Taken together, these findings suggest
that CHEM and uncertainty quantification capture complementary aspects of
reconstruction quality, and that hallucination-specific detection can reveal
structured artifacts that model uncertainty alone does not spatially
isolate.


\section{Conclusions and Discussion}

In this paper, we introduced a new theoretical framework to quantify and understand hallucinations in image processing tasks. Using conformalized quantile regression, we developed a model-agnostic, distribution-free hallucination measure (CHEM) based on wavelet and shearlet representations. We theoretically demonstrated that CHEM exhibits exponential sensitivity to artifacts and is closely related to the MSE error. In particular, within hallucination-prone regions, CHEM can be lower bounded by MSE up to certain constant factors. This observation motivated our focus on analyzing MSE error and offers additional insight into why neural networks may generate artifacts.
By studying U-shaped 
architectures, we pinpointed various factors that lead to these unrealistic 
artifacts. We carried out our experiments on two distinct imaging tasks and applied a range of deep learning models for the analysis, highlighting the utility and importance of CHEM in detecting and quantifying artifacts.

Several directions remain open for future work. The proposed approach 
relies on predefined dictionaries; for images from diverse domains, 
automatically adapting the dictionary to the relevant image class could 
enhance the efficiency of hallucination detection. Furthermore, our experiments 
reveal a tradeoff between hallucination and performance, whose underlying 
causes merit deeper investigation. Finally, extending this analysis to 
domains where reconstruction consistency is critical, such as medical 
imaging, is a natural and important next step.

\section*{Acknowledgement}
\addcontentsline{toc}{section}{Acknowledgement}

J. Li gratefully acknowledges support from the project CONFORM, funded by the 
German Federal Ministry of Education and Research (BMBF), as well as from the 
German Research Foundation under the Grant DFG-SPP-2298. Furthermore, J. Li 
acknowledges additional support from the project ``Next Generation AI Computing 
(gAIn)'', funded by the Bavarian Ministry of Science and the Arts and the Saxon 
Ministry for Science, Culture, and Tourism.
I. Rosellon-Inclan gratefully acknowledges support from the project CONFORM, 
funded by the German Federal Ministry of Education and Research (BMBF).
The work of G. Kutyniok was supported in part by the Konrad Zuse School of 
Excellence in Reliable AI (DAAD), the Munich Center for Machine Learning (MCML), 
as well as the German Research Foundation under Grants DFG-SPP-2298, KU 1446/31-1 
and KU 1446/32-1. Furthermore, G. Kutyniok acknowledges additional support from 
the project ``Next Generation AI Computing (gAIn)'', funded by the Bavarian 
Ministry of Science and the Arts and the Saxon Ministry for Science, Culture, 
and Tourism, as well as by the Hightech Agenda Bavaria. The work of J. Starck was supported by the TITAN ERA Chair project (contract no. 101086741) within the Horizon Europe Framework Programme of the European Commission, and additionally by the Agence Nationale de la Recherche (ANR-22-CE31-0014-01 TOSCA).

 {
     \small
     \bibliographystyle{ieeenat_fullname}
     \bibliography{references}
 }


\section{Theoretical Proofs}\label{app:proofs}

\subsection{Notations}

We denote $\|\cdot\|_p$, $p\in[1,\infty]$ as the vector norm. Let $\Omega \subset \RR^d$ be a compact set. The $L_p(\Omega)$ space contains measurable functions that have a finite $L_p$ norm
$    \|f\|_{L_p(\Omega)} = \left( \int_{\Omega} |f(x)|^p dx \right)^{1/p} <\infty$. For $p=\infty$, we define $\|f\|_{L_\infty(\Omega)}:=\operatorname{ess\,sup}|f|$.

Since $L_2(\Omega)$ is a Hilbert space, we denote its inner product as
$\langle f_1, f_2 \rangle := \int_{\Omega} f_1(x) f_2(x) d x$.
Throughout our proofs, we usually use $\Omega = [-1,1]^d$ and, for simplicity, write $\|\cdot\|_{L_p}:= \|\cdot\|_{L_p(\Omega)}$.
We consider $\X, \Y \subset L_2(\Omega)$ are compact sets of $L_2(\Omega)$.
The modulus of continuity of a general continuous mapping $\M:L_\infty([-1,1]^d) \rightarrow L_\infty([-1,1]^d)$ is defined as
\begin{align*}
    \omega_{\M}(r) :=  \sup_{\substack{f_1,f_2 \in L_2([-1,1]^d),\\ \|f_1 - f_2\|_{L_\infty} \leq r
     }} \|\M(f_1) - \M(f_2) \|_{L_\infty}.
\end{align*}

\subsection{Proof of Proposition~\ref{prop:hoeffding}}
\label{proof:hoeffding}
We use Hoeffding's inequality to establish Proposition~\ref{prop:hoeffding}.
\begin{proof}[Proof of Proposition~\ref{prop:hoeffding}]
    Let $$S_M = \sum_{m=1}^M H^{\theta}( X_m, Y_m)_j.$$
    Since $0\leq H^{\theta}( X, Y)_j\leq \theta$, Hoeffding's inequality implies that
    \begin{align*}
        \PP \left( |S_M - \EE S_M|\geq t \right) \leq 2 \exp\left\{-\frac{2 t^2}{M\theta^2}\right\}.
    \end{align*}
    Substituting $s = t/M$ into the above inequality, we further have
    \begin{align*}
        \PP \left( \left|\frac{1}{M}S_M - \EE \left[ \frac{1}{M} S_M\right]\right|\geq s \right) \leq 2 \exp\left\{-\frac{2M s^2}{\theta^2}\right\}.
    \end{align*}
    Let $s = \sqrt{\frac{\theta^2\log(2/\delta)}{2M}}$. Notice that $\EE \left[ \frac{1}{M} S_M\right] = \EE H^\theta(X,Y)_j = \operatorname{H}^\theta(\Phi)_j$. Hence we finally obtain
    \begin{align*}
        \PP \left( \left|\frac{1}{M}S_M - \EE \left[ \frac{1}{M} S_M\right]\right|\geq \sqrt{\frac{\theta^2\log(2/\delta)}{2M}} \right) \leq \delta.
    \end{align*}
    The proof is completed by applying similar steps for $\operatorname{H}^\theta(\Phi)$.
    
\end{proof}

\subsection{Proof of Theorem~\ref{thm:chem_sensitivity}}
\label{proof:chem_sensitivity}
Before proving Theorem~\ref{thm:chem_sensitivity}, we begin by introducing the precise definitions and key properties of the sub-Gaussian distributions that will be employed.
Sub-Gaussian random variables are distinguished by having tails that decay at least as rapidly as those of a Gaussian distribution. Their most precise characterization is given in terms of the $\psi_2$-norm.
\begin{definition}[Sub-Gaussian Norm]
The sub-Gaussian norm of a random variable $X$, denoted by $\|X\|_{\psi_2}$, is defined as:
\begin{equation*}
    \|X\|_{\psi_2} = \inf \left\{ K > 0 : \mathbb{E}\left[\exp\left(X^2/K^2\right)\right] \leq 2 \right\}.
\end{equation*}
\end{definition}
If $\|X\|_{\psi_2}$ is finite, $X$ satisfies the tail inequality $\mathbb{P}(|X| \geq t) \leq 2 \exp(-c t^2 / \|X\|_{\psi_2}^2)$ for some constant $c > 0$.
The extension to multivariate distributions is achieved by considering projections onto the unit sphere.
\begin{definition}[Sub-Gaussian Vector]
A random vector $\eta \in \mathbb{R}^d$ is sub-Gaussian if for every $v \in \mathbb{R}^d$, the scalar random variable $\langle \eta, v \rangle$ is sub-Gaussian. The sub-Gaussian norm of the vector $\eta$ is defined as:
\begin{equation}
    \|\eta\|_{\psi_2} = \sup_{\|v\|_2=1} \|\langle \eta, v \rangle\|_{\psi_2}.
\end{equation}
\end{definition}
A crucial property used in our analysis is that sub-Gaussianity is preserved under linear transformations. Specifically, for a Parseval frame $\Psi$, the coefficients $(W \eta)_j = \langle \eta, w_j \rangle$ are sub-Gaussian with norm $\|(W \eta)_j\|_{\psi_2} \leq \|\eta\|_{\psi_2}$.
\begin{proof}[Proof of Theorem~\ref{thm:chem_sensitivity}]
    Let $\hat{Z} := W Z$ for any $Z\in \RR^{t_2}$ and $W:= [w_1, \dots, w_N]$ with $w_j \in \RR^{t_2}$.
    We first show that $(\hat{\eta})_j = \langle w_j, \eta\rangle$ is sub-Gaussian. Recall that for a Parseval frame, the frame identity $\|Z\|_2^2 = \sum_{j \in [N]} |\langle Z, w_j \rangle|^2$ holds for any $Z \in \mathbb{R}^{t_2}$. By choosing $Z = w_k$, we observe that
    \begin{equation*}
        \|w_k\|_2^2 = \|w_k\|_2^4 + \sum_{j \neq k} |\langle w_k, \psi_j \rangle|^2 \implies \|w_k\|_2^2 \geq \|w_k\|_2^4,
    \end{equation*}
    which implies $\|w_j\|_2 \leq 1$ for all $j$. 
    
    Using the homogeneity of the sub-Gaussian norm, we have
    \begin{align*}
        \| \langle \eta, w_j \rangle \|_{\psi_2} &=\left \| \|w_j\|_2 \left\langle \eta, \frac{w_j}{\|w_j\|_2} \right\rangle \right \|_{\psi_2}
        \\
        &= \|w_j\|_2 \cdot \left\| \left\langle \eta, \frac{w_j}{\|w_j\|_2} \right\rangle \right\|_{\psi_2}.
    \end{align*}
    Since $v_j = \frac{w_j}{\|w_j\|_2}$ is a unit vector, it follows from the definition of the vector sub-Gaussian norm that
    \begin{equation*}
        \left\| \langle \eta, v_j \rangle \right\|_{\psi_2} \leq \sup_{v \in S^{d-1}} \|\langle \eta, v \rangle\|_{\psi_2} = \|\eta\|_{\psi_2}.
    \end{equation*}
    Combining these results and utilizing the fact that $\|\psi_j\|_2 \leq 1$, we obtain
    \begin{equation*}
        \|(\hat \eta)_j\|_{\psi_2} \leq \|w_j\|_2 \|\eta\|_{\psi_2} \leq \|\eta\|_{\psi_2}.
    \end{equation*}
    Hence, $\|(\hat \eta)_j\|_{\psi_2} \leq \|\eta\|_{\psi_2}$ holds for any $j$.

    Now we are ready to examine the characteristics of the CHEM.
    Notice that
	the CHEM 
    $$\operatorname{H}:= \frac{1}{\hat t_2}\sum_{j=1}^{\hat t_2}\operatorname{dist}(\hat Y_j, B_\alpha(\widehat{\Phi(X)})_j)$$ fails to detect a hallucination if $\operatorname{H} = 0$, which occurs if and only if the ground truth $Y$ is contained within the conformal set $B_\alpha(\widehat{\Phi(X)})$ centered at the prediction $\Phi(X)$. In the $\Psi$-domain, this implies:
	\begin{equation*}
		|(\hat Y)_j - (\widehat{\Phi(X)})_j| \leq \hat{R}(X)_j, \quad \forall j.
	\end{equation*}
	Substituting the decomposition $\Phi(X) = Y + h + \eta$, the failure condition for a specific index $j$ becomes:
	\begin{equation*}
		|(\hat Y)_j - (\hat Y + \hat h + \hat \eta)_j| = |(\hat h)_j + (\hat \eta)_j| \leq \hat{R}(X)_j.
	\end{equation*}
	By the reverse triangle inequality, the deviation $|(\hat h)_j + (\hat \eta)_j|$ is at least $|(\hat h)_j| - |(\hat \eta)_j|$. Therefore, for failure to occur at index $j$, the residual must compensate for the hallucination:
	\begin{equation*}
		|(\hat \eta)_j| \geq |(\hat h)_j| - \hat{R}(X)_j.
	\end{equation*}
	This must hold for all $j$, including the index $j^* = \arg \max_j |(\hat h)_j|$ where the hallucination is most prominent.

    Recall that $W$ is a Parseval frame and $h$ is $s$-sparse. It is easy to derive the inequality
	\begin{equation*}
		|(\hat h)_{j^*}| \geq s^{-1/2} \|h\|_2.
	\end{equation*}
    since $\|x\|_s \leq \sqrt{s} \|x\|_\infty$ for any $x\in\RR^s$ and $\|W h\|_2 = \|h\|_2$.
	
	The probability of the failure event is bounded by the probability that the residual at $j^*$ is large enough to ``mask'' the hallucination
	\begin{align*}
		P(\operatorname{H} = 0 \mid h) 
		&\leq P\left( |(\hat \eta)_{j^*}| \geq s^{-1/2} \|h\|_2 - \hat{R}(X)_{j^*} \right) .
	\end{align*}
	Substituting the sub-Gaussian bound yields
	\begin{equation*}
		P(\operatorname{H} = 0 \mid h) \leq 2 \exp \left( - \frac{c\left( s^{-1/2} \|h\|_2 - \hat{R}(X)_{j^*} \right)^2}{\|\eta\|_{\psi_2}^2} \right).
	\end{equation*}
\end{proof}

\subsection{Proof of Theorem~\ref{thm:hall}}
\label{proof:hall}
\begin{proof}[Proof of Theorem~\ref{thm:hall}]
By the triangle inequality, we obtain
\begin{align*}
    \HH_{\S} &= \error_{\S}\left( \hat{Y} , \proj (\hat{Y}) \right) \\
    &\geq \error_{\S}\left( \hat{Y} ,\widehat{\Phi(X)} \right) -\error_{\S}\left( \proj (\hat{Y}) , \widehat{\Phi(X)} \right) \\
    &\geq (1-\mu) \error_{\S}(\hat{Y} , \widehat{\Phi(X)}) \\
    &= (1-\mu) R_{\S} \error(\hat{Y} , \widehat{\Phi(X)}).
\end{align*}
Similarly, we can derive the upper bound of $\HH_{\Sc}$ using Assumption~\ref{ass:highfreq_hall}
\begin{align*}
    \HH_{\Sc} &= \error_{\Sc}\left( \hat{Y} , \proj (\hat{Y}) \right) \\
    &\leq \error_{\Sc}\left( \hat{Y} , \widehat{\Phi(X)} \right) + \error_{\Sc}\left( \proj (\hat{Y}) , \widehat{\Phi(X)} \right) \\
    &=(1+\nu) \error_{\Sc}\left( \hat{Y} , \widehat{\Phi(X)} \right).
\end{align*}
Thus,
\begin{align*}
    \HH_{\Sc} 
    &\leq (1+\nu) \error_{\Sc}(\hat{Y} , \widehat{\Phi(X)}) \\
    &= (1+\nu)(1-R_{\S}) \error(\hat{Y} , \widehat{\Phi(X)}).
\end{align*}
\end{proof}


\subsection{U-shaped architectures}

U-Net represents an architecture of convolutional neural networks primarily designed to solve image segmentation tasks. 
One way to improve the performance of U-shaped architectures is to add a bottleneck between the encoding path and the decoding path \cite{zhou_unet_2018, chen_transunet_2021, cao_swin-unet_2023, li2022convolutional}. The bottleneck is used to learn nonlinear transformations of the encoder's deepest features before passing them to the decoder. 
U-Net corresponds to the special case where the bottleneck is the identity mapping.

\subsection{Formal definition of U-shaped networks}

For a 1D convolution with a kernel $w \in \mathbb{R}^s$ and a stride $t$, the convolution operation for $x \in \mathbb{R}^d$ is defined as
$(w *_t x)_i = \sum_{k=1}^{s} w_k\, x_{(i-1)t+k}
$, $ i = [ \mathcal{D}(d,t)]$,
where $\mathcal{D}(d,t) := \lceil d/t \rceil$ and we set $x_i = 0$ for $i > d$. 
Obviously, there exists a matrix $T^{w,t} \in \mathbb{R}^{\mathcal{D}(d,t)\times d}$ determined by the kernel $w$ and the stride $t$ such that $T^{w,t}x = w *_t x$.
Subsequently, we describe a convolutional block consisting of layers $L$ as $\Phi = \T_{L} \circ \sigma \circ \T_{L-1} \circ \cdots \circ \sigma \circ \T_1$, 
where $\T_\ell(z) := T^{w_\ell,t_\ell} z + b_\ell$, $\ell\in [ L]$ is associated with a kernel $w_\ell$, a stride $t_\ell$, and a bias $b_\ell$. We use $K$ to denote the total number of parameters in all kernels and bias. In our theoretical analysis, we focus mainly on 1D convolutional blocks. It is important to note that in certain situations, convolutional blocks for processing 2D or higher dimensional tensors, which are more commonly used in image processing tasks, can be simplified into 1D cases \cite{li2024approximation}.

Before defining U-shaped networks, we first introduce multichannel convolutional blocks.
\begin{definition}[Multichannel Convolutional Blocks]\label{1dcnn}
Given an input $x\in \RR^d$, a $J$ layer multi-channel CNN block $\phi: \RR^d \to  \RR^{d_J \times n_J}$ with filters $\{ w^{(j)}_{\ell,i}\in \RR^{s_j}\}_{j=1}^J$, filter size $\{s_j\geq 2\}_{j=1}^J$, stride $\{t_j\geq 1\}_{j=1}^J$ and channel $\{n_j\geq 1\}_{j=1}^J$,
 is defined iteratively by
 \begin{small}
     \begin{align*}
     \phi^{(1)}( x)_{\ell} &= \sigma \left ( \ w^{(1)}_{\ell,1}*_{t_1}  x+ b_\ell^{(1)} \bm{1}_{d_1}  \right ), \ell \in [n_1], \\
    \phi^{(j)}( x)_{\ell} &=\sigma\left(\sum_{i=1}^{n_{j-1}}\ w^{(j)}_{\ell,i}*_{t_j} \phi^{(j-1)}( x)_i+b^{(j)}_\ell \bm{1}_{d_j} \right), \ell \in [n_j], \\
    \phi( x)_{\ell} &=  \sum_{i=1}^{n_{J-1}}\ w^{(J)}_{\ell,i}*_{t_J} \phi^{(J-1)}( x)_i+b^{(J)}_\ell \bm{1}_{d_J}, \ell \in [n_J],
 \end{align*}
 \end{small}
where $b^{(j)}_\ell \in \mathbb{R}$ are biases and  $d_j=\D(d_{j-1},t_j)$. 
\end{definition}


\begin{definition}[U-shaped networks]
    A U-shaped network is defined as the composition of an encoding path $\Phi_{\en}$, a bottleneck $\Phi_{\app}$, and a decoding path $\Phi_{\de}$ where $\Phi_{\en}$, $\Phi_{\app}$, and $\Phi_{\de}$ are multichannel convolutional blocks. The summation of layers of all the multichannel convolutional blocks is called the depth/layers of the U-shaped networks, and similarly the summation of number of parameters of all the multichannel convolutional blocks is called the total number of parameters of the U-shaped networks. We denote $\unet(L,K)$ as the collection of all U-shaped networks with layers no more than $L$ and the total number of parameters at most $K$.
\end{definition}
\begin{figure}
    \centering
    \includegraphics[width=0.7\linewidth]{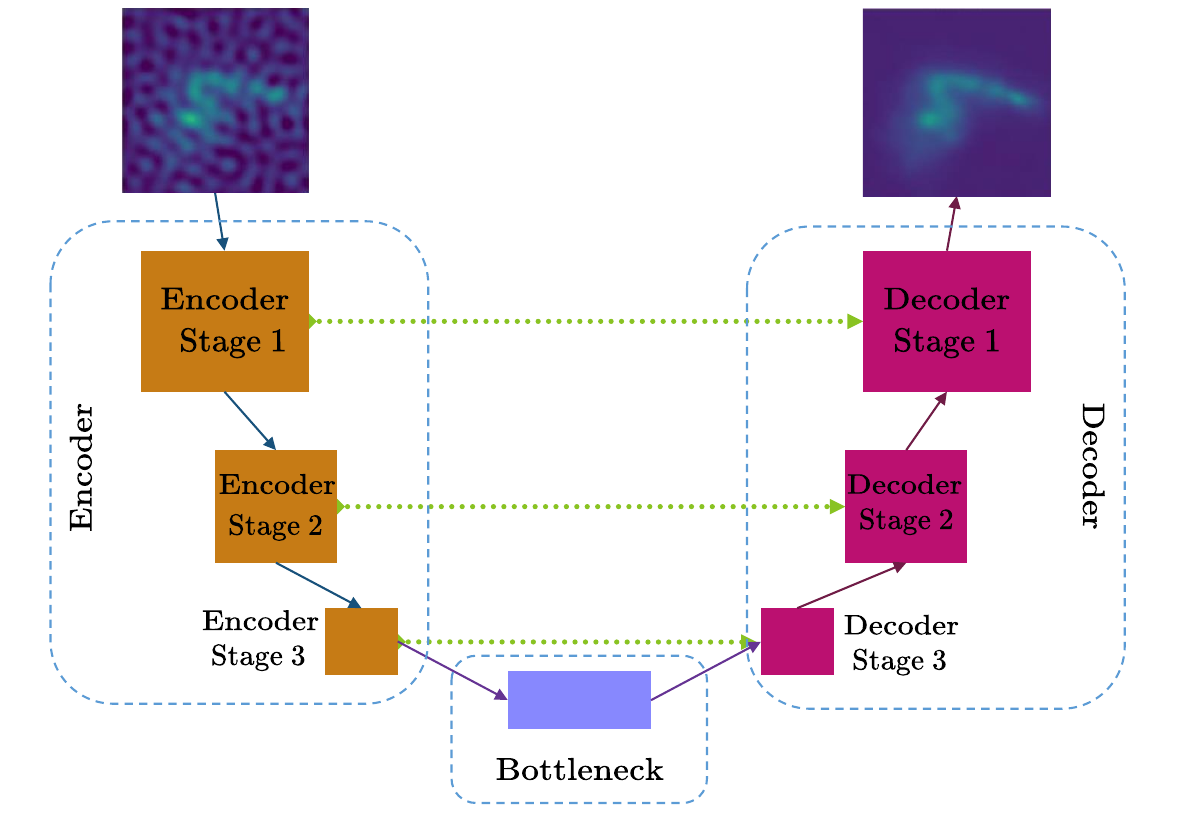}
    \caption{U-shaped network architectures. The foundational components, namely the encoder stages, decoder stages, and the bottleneck,  can be built from convolutional layers, attention layers, or transformer architectures.}
    \label{fig:UshapedNet}
\end{figure}
For theoretical ease, the outputs of convolutional blocks which are 1D in each channel are flattened into a vector and serve as the inputs to the subsequent blocks. To prove the main results, we also restrict convolutions for defining $
\unet(L,K)$ to the case of $s_j =t_j = 2$ for all convolution operations. Other convolutional settings are discussed in \cite{li2024approximation}.
\begin{Remark}
    \textbf{Skip connections in U-shaped networks: }In our definition, we did not consider skip connections, which are also a crucial component of U-Net. This omission is due to the fact that \cite{guhring2020error} demonstrates that any fully connected neural network with skip connections can be equivalently represented by a counterpart lacking skip connections. Therefore, our findings are applicable to U-shaped networks that include skip connections.
\end{Remark}
\begin{Remark}
    \textbf{U-shaped networks constructed with self-attention layers and \\  transformers:} Self-attention mechanisms and transformers have emerged as pivotal elements in deep learning-based image processing models. As illustrated by \cite{cordonnier2020relationship}, self-attention mechanisms are capable of approximating any convolution operation. The transformer, which is fundamentally composed of fully connected layers and self-attention, further emphasizes this adaptability, since \cite{li2024approximation} indicates that convolutional neural networks can represent any fully connected layers. These insights imply that our main results may also extend to U-shaped networks constructed using transformers or self-attention layers.
\end{Remark}

\subsection{Proof of Theorem~\ref{thm:main}}
\label{proof:main}



We start with the lemma below, which demonstrates that any continuous function can be closely approximated by a polynomial.

\begin{lemma}\label{lem:linear_operator}
    Let $f \in \C([-1,1]^d)$.
    There exists a linear operator $V_m : \C([-1,1]^d) \rightarrow \Pi_{m}$ such that
    \begin{align*}
        \| f - V_m f \|_{L_\infty([-1,1]^d)} \leq \frac{5d}{4} \omega_{f}\left(\frac{2}{m};[-1,1]^d \right).
    \end{align*}
    In addition, we have $\|V_m\|_{L_\infty([-1,1]^d)}=1$.
\end{lemma}
\begin{proof}
    Let $z = 2x - 1$ where $x \in [0,1]^d$. We denote $g(x) = f(2x-1) $. Then the function $g $ is a continuous function in $ \C([0,1]^d)$ and there exists a linear operator $V_m : \C([0,1]^d) \rightarrow \Pi_{m}$,
    defined as
    \begin{small}
        \begin{align*}
        V_m(g)(x) &:= \sum_{k_1=0}^{m} \cdots \sum_{k_d=0}^{m} \binom{m}{k_1}\cdots \binom{m}{k_d} g\left( \frac{k_1}{m}, \dots, \frac{k_d}{m} \right) \\ 
        &\quad\times \prod_{i=1}^d \left( x_i^{k_i}(1-x_i)^{m-k_i} \right),
    \end{align*}
    \end{small}
    such that \cite[Theorem 3.1]{schultz1969multivariate}
    \begin{align*}
        \left|g(x) - V_m (g)(x)\right| \leq \frac{5d}{4}\omega_g\left(\frac{1}{m};[0,1]^d \right), \quad \forall x \in [0,1]^d.
    \end{align*}
    Since $g(x) = f(2x-1)$ and $z = 2x-1$, we can rewrite the above bound as
    \begin{align}\label{eq:constructionVm}
        \left| f(z) - \tilde{V}_m(f)(z) \right| \leq \frac{5d}{4}\omega_g\left(\frac{1}{m};[0,1]^d \right), \quad \forall z \in [-1,1]^d,
    \end{align}
    where 
    \begin{small}
        \begin{align*}
        \tilde{V}_m(f)(z) &= \sum_{k_1=0}^{m} \cdots \sum_{k_d=0}^{m}\binom{m}{k_1}\cdots \binom{m}{k_d} \\
        &\quad\times f\left( \frac{2k_1-m}{m}, \dots, \frac{2k_d-m}{m} \right) \\
        &\quad\times \prod_{i=1}^d \left( \left(\frac{z_i+1}{2}\right)^{k_i}\left(\frac{1-z_i}{2}\right)^{m-k_i} \right).
    \end{align*}
    \end{small}
    Using the definition of modulus of continuity, we further have
    \begin{small}
        \begin{align}\label{eq:g_f_modulus}
    \begin{aligned}
        &\omega_g\left(\frac{1}{m};[0,1]^d \right) \\
        &= \left \{ |g(x_1) - g(x_2) |: \|x_1 - x_2\|_{2} \leq \frac{1}{m} , x_1,x_2 \in [0,1]^d  \right\} \\
        &= \left \{ |f(z_1) - f(z_2) |: \|z_1 -z_2\|_{2} \leq \frac{2}{m} , z_1,z_2 \in [-1,1]^d  \right\} \\
        &= \omega_f \left( \frac{2}{m};[-1,1]^d  \right).
    \end{aligned}
    \end{align}
    \end{small}
    
    By substituting \eqref{eq:g_f_modulus} into \eqref{eq:constructionVm}, we get the claimed upper bound.
    
    For estimating $\|\tilde{V}_m\|_{L_\infty}$, it is straightforward to see that $\|\tilde{V}_m\|_{L_\infty}=1$ by multi-binomial theorem
    \begin{align*}
        &\|\tilde{V}_m (f)\|_{L_\infty([-1,1]^d)} \\
        &\leq \|f\|_{L_\infty([-1,1]^d)} \sum_{k_1=0}^{m} \cdots \sum_{k_d=0}^{m}\binom{m}{k_1}\cdots \binom{m}{k_d}
         \\
        &\quad \times \prod_{i=1}^d \left( \left(\frac{z_i+1}{2}\right)^{k_i}\left(\frac{1-z_i}{2}\right)^{m-k_i} \right)\\
         &= \|f\|_{L_\infty([-1,1]^d)}.
    \end{align*}
\end{proof}

\begin{theorem}\label{thm:discretization}
    Let $\X, \Y \subset \C([-1,1]^d)$ and $\M :\X \rightarrow \Y$. Suppose that Assumption~\ref{ass:Lipschitz} and Assumption~\ref{ass:Y} hold. Then
    we have for any $f \in \X$,
    \begin{align*}
        &\|\M(f) - V_m \circ \M \circ V_m (f) \|_{L_\infty } \\
        &\leq 6L_{\M}d^2 \omega_{f}\left(\frac{2}{m};[-1,1]^d \right) + \frac{5L_{\Y}d}{2m}.
    \end{align*}
\end{theorem}

\begin{proof}[Proof of Theorem~\ref{thm:discretization}]
    To begin with, we present the subsequent error decomposition
    \begin{align}\label{eq:error_decomposition}
    \begin{aligned}
        &\|\M(f) - V_m \circ \M \circ V_m (f) \|_{L_\infty } \\
        &\leq \underbrace{\|\M(f) -  \M \circ V_m (f) \|_{L_\infty } }_{(i)} \\
        &\quad + \underbrace{\| \M \circ V_m (f) - V_m \circ \M \circ V_m (f) \|_{L_\infty }}_{(ii)}.
    \end{aligned}
    \end{align}
    For the term $(i)$, we can derive 
    \begin{align}\label{eq:i}
        (i) \leq L_{\M}\|f
        -   V_m (f) \|_{L_\infty }  \leq  \frac{5L_{\M}d}{4} \omega_{f}\left(\frac{2}{m};[-1,1]^d \right) ,
    \end{align}
    where we use Lemma~\ref{lem:linear_operator} and Lipschitz property of $\M$ in Assumption~\ref{ass:Lipschitz}.
    
    Since $\M \circ V_m(f) \in \C([-1,1]^d)$,
    applying Lemma~\ref{lem:linear_operator} again to the second term $(ii)$, we derive
    \begin{align}\label{eq:ii}
    \begin{aligned}
        &\|  \M \circ V_m(f) - V_m \circ \M \circ V_m (f) \|_{L_\infty([-1,1]^d) } \\
        &\leq \frac{5d}{4} \omega_{ \M \circ V_m(f)}\left(\frac{2}{m};[-1,1]^d \right). 
    \end{aligned}
    \end{align}

    Now we need to estimate $\omega_{ \M \circ V_m(f)}$. Lemma~\ref{lem:linear_operator},
    Assumption~\ref{ass:Lipschitz}, and Assumption~\ref{ass:Y} jointly imply that
        \begin{align*}
        &\left|\M \circ V_m(f)(x_1) - \M \circ V_m(f)(x_2) \right| \\
        &\leq \left|\M \circ V_m(f)(x_1) - \M (f)(x_1) \right| \\
        &\quad + \left|\M (f)(x_1) - \M(f)(x_2) \right| \\
         &\quad+ \left|\M (f)(x_2) - \M \circ V_m(f)(x_2) \right| \\
        &\leq 2 \| \M \circ V_m(f) - \M (f) \|_{L_\infty} + L_{\Y}\|x_1 - x_2\|_{2} \\
        &\leq 2 L_{\M} \|  V_m(f) - f \|_{L_\infty} + L_{\Y}\|x_1 - x_2\|_{2} \\
        &\leq \frac{10L_{\M}d}{4} \omega_{f}\left(\frac{2}{m};[-1,1]^d \right) + L_{\Y}\|x_1 - x_2\|_{2}.
    \end{align*}

    Substituting the above inequality into the definition of $\omega_{ \M \circ V_m(f)}$, we obtain
    \begin{align*}
        &\omega_{ \M \circ V_m(f)}\left(\frac{2}{m};[-1,1]^d \right) \\
        &= \Big \{ |\M \circ V_m(f)(x_1) - \M \circ V_m(f)(x_2) |: \|x_1 - x_2\|_{2} \leq \frac{2}{m} ,\\
        &\quad\quad x_1,x_2 \in [-1,1]^d  \Big\}  \\
        &\leq \frac{10L_{\M}d}{4} \omega_{f}\left(\frac{2}{m};[-1,1]^d \right) + \frac{2L_{\Y}}{m}.
    \end{align*}
    Hence, combining \eqref{eq:ii} with the above estimation of $\omega_{ \M \circ V_m(f)}$, the second term $(ii)$ is bounded by
    \begin{align}\label{eq:iii}
        (ii) \leq  \frac{25L_{\M}d^2}{8} \omega_{f}\left(\frac{2}{m};[-1,1]^d \right) + \frac{5L_{\Y} d}{2m}.
    \end{align}
    We finally get the upper bound by substituting \eqref{eq:i} and \eqref{eq:iii} into \eqref{eq:error_decomposition}
    \begin{align*}
        &\|\M(f) - V_m \circ \M \circ V_m (f) \|_{L_\infty }\\
        &\leq 6L_{\M}d^2 \omega_{f}\left(\frac{2}{m};[-1,1]^d \right) + \frac{5L_{\Y}d}{2m}.
    \end{align*}
    We finish the proof.
\end{proof}

Theorem~\ref{thm:discretization} demonstrates that as $m$ approaches infinity, the expression $\M \rightarrow V_m \circ \M \circ V_m$ converges uniformly on $\X$. The concept of $V_m$ implies that in practical applications, when handling data, it is necessary to convert a continuous signal into a finite-dimensional vector. Since $V_m$ produces polynomials up to degree $m$, which constitutes a finite-dimensional space, this space is consequently learnable. Thus, we consider the use of U-shaped networks to approximate $V_m \circ \M \circ V_m$. If this approximation approaches zero, it means that U-architectures possess the universal approximation property.

We now introduce some fundamental concepts that will be used later. We define Legendre polynomials as 
\begin{align*}
    P_n(x) : = \frac{(-1)^n\sqrt{n+1/2}}{2^nn!} \left( \frac{d}{dx} \right)^n \left\{ (1-x^2)^n \right \}.
\end{align*}
For $x = (x_1, \dots, x_d)^\top \in \RR^d$, we define
\begin{align*}
    P_{\bm k}(x) = \prod_{j=1}^d P_{k_j}(x_j),
\end{align*}
where $\bm k := (k_1, \dots, k_d) \in \RR^d$. Notice that we have
\begin{align*}
    \langle P_{\bm k_1}, P_{\bm k_2} \rangle = \delta_{\bm k_1,\bm k_2},
\end{align*}
where $\delta_{\bm k_1,\bm k_2}=1$ if $\bm k_1 = \bm k_2$ and zero otherwise.
Without loss of generality, we replace the multi-index with the the scalar index, i.e., $\{ P_{\bm k}: \bm k \in \{0,1,2,\dots , m\}^d \}  =  \{ P_{k} \}_{k=1}^{(m+1)^d}$. Obviously $\{ P_{k} \}_{k=1}^{(m+1)^d}$ forms a basis of the polynomial space $\Pi_{m}$.

We introduce an isometric isomorphism mapping $\phi: \Pi_m \rightarrow \RR^{(m+1)^d}$:
\begin{align*}
    \phi(Q) := (\langle Q, P_1 \rangle,\dots, \langle Q, P_t \rangle)^\top,
\end{align*}
where $t := (m+1)^d$ and $Q \in \Pi_m$.
Then we can see that for any $\hat{Q}\in \RR^t$
\begin{align*}
    \phi^{-1}(\hat{Q}) := \sum_{i=1}^t \hat{Q}_i P_i,
\end{align*}
satisfies
\begin{align*}
\begin{aligned}
    \phi^{-1}\phi(Q) &= Q \\
    \|Q\|_{L_2}^2 &= \|\hat{Q}\|_2^2.
\end{aligned}
\end{align*}
Hence we can rewrite $V_m \circ \M \circ V_m$ as
\begin{align*}
    V_m \circ \M \circ V_m  = \phi^{-1} \circ \left( \phi \circ V_m \circ \M \circ \phi^{-1} \right) \circ \phi \circ V_m (f),
\end{align*}
for further analysis. 
Here we have that $\phi \circ V_m \circ \M \circ \phi^{-1} :\RR^t \rightarrow \RR^t$ is a mapping between finite dimensional spaces and $\phi$ encodes the information of the function $V_m f$ into a vector and $\phi^{-1}$ decodes the information from a vector into a continuous function.

The following lemma is a useful property for characterizing norms of polynomials.

\begin{lemma}[\cite{mhaskar1997neural}]\label{lem:mhaskar}
    Let $p,q \in [1,\infty]$, then for any $m\in \NN$ and $Q \in \Pi_m$, we have
    \begin{align*}
        \|Q\|_{L_p([-1,1]^d)} \leq c \cdot m^{2d \max\{ \frac{1}{q}-\frac{1}{p},0 \}} \|Q\|_{L_q([-1,1]^d)},
    \end{align*}
    where $c$ is a constant and independent of $m$.
\end{lemma}

The following lemma shows that $\phi \circ V_m \circ \M \circ \phi^{-1}$ is Lipschitz.
\begin{lemma}\label{lem:approximatorLip}
    Suppose Assumption~\ref{ass:Lipschitz} holds.
    The mapping $ \phi \circ V_m \circ \M \circ \phi^{-1}: \RR^t \rightarrow \RR^t $ is Lipschitz, that is,
    \begin{align*}
        \begin{aligned}
            &\|\phi \circ V_m \circ \M \circ \phi^{-1}(w_1) - \phi \circ V_m \circ \M \circ \phi^{-1}(w_2) \|_\infty \\
            &\leq c L_{\M} m^{d}  \| w_1 - w_2 \|_{
            2} , 
        \end{aligned}
    \end{align*}
    for any $ w_1,w_2 \in \RR^t$.
\end{lemma}
\begin{proof}
    Let $w_1,w_2 \in \RR^t$. Then we have
    \begin{align*}
        \begin{aligned}
            &\|\phi \circ V_m \circ \M \circ \phi^{-1}(w_1) - \phi \circ V_m \circ \M \circ \phi^{-1}(w_2) \|_\infty \\
            &\leq \|\phi \circ V_m \circ \M \circ \phi^{-1}(w_1) - \phi \circ V_m \circ \M \circ \phi^{-1}(w_2) \|_2 \\
            &= \| V_m \circ \M \circ \phi^{-1}(w_1) -  V_m \circ \M \circ \phi^{-1}(w_2) \|_{L_2} \\
            &\leq c\| V_m \circ \M \circ \phi^{-1}(w_1) -  V_m \circ \M \circ \phi^{-1}(w_2) \|_{L_\infty} \\
            &\leq c L_{\M} \| \phi^{-1} (w_1) - \phi^{-1}(w_2) \|_{L_
            \infty} \\
            &\leq c L_{\M} m^d \| \phi^{-1} (w_1) - \phi^{-1}(w_2) \|_{L_
            2} \\
            &= c L_{\M} m^{d}  \| w_1 - w_2 \|_{
            2} ,
        \end{aligned}
    \end{align*}
    where in the first step we use $\|z\|_\infty \leq \|z\|_2$ for any $z\in \RR^t$, in the second step we use the definition of $\phi$, in the third step we use Lemma~\ref{lem:mhaskar}, in the fourth step we use Assumption~\ref{ass:Lipschitz} and Lemma~\ref{lem:linear_operator}, in the fifth step we use Lemma~\ref{lem:mhaskar}, and in the last step we use the definition of $\phi^{-1}$.
\end{proof}

For any Lipschitz mapping, it is possible to identify a CNN where the approximation error remains sufficiently small. This conclusion is presented in the following lemma.

\begin{lemma}\label{lem:approximator}
    Let $L,N \in \NN_{+}$ with $N > t^{t/(t-1)}$. Suppose Assumption~\ref{ass:Lipschitz} holds. Then there exists a CNN $\Phi$ with depth $O(L\log(t N))$ and total number of parameters $O(Lt^2N^2)$ such that
    \begin{align*}
        &\| \phi \circ V_m \circ \M \circ \phi^{-1}(w) - \Phi(w) \|_{\infty} \\
        &\leq C L_{\M} m^{d} (N^2L^2 \log(N))^{-1/t},
    \end{align*}
    for any $w\in [-1,1]^t$ where $C$ is a positive constant.
\end{lemma}
\begin{proof}
    Let $L,N \in \NN_{+}$ with $N > t^{t/(t-1)}$. Given a Lipschitz function $f$ with Lipschitz constant $\lambda$, \cite[Corollary 1.3]{shen2022optimal} implies that there exists a ReLU neural network $\psi$ with depth $O(L)$ and width $O(N)$ such that
    \begin{align*}
        \|f - \psi \|_{L_\infty([-1,1]^d)} \leq C \lambda (N^2L^2 \log(N))^{-1/t},
    \end{align*}
    where $C$ depends only on $d$.
    
    Consider $f_i$ as Lipschitz functions, where $i\in[t]$.
    Using \cite[Proposition 4.1]{yang2025optimal}, we can then approximate the following mapping
    \begin{align*}
    F(x) :=
        \begin{pmatrix}
            f_1(x)\\f_2(x)\\ \cdots \\ f_t(x)
        \end{pmatrix}
    \end{align*}
    by using a ReLU neural network $\Psi :\RR^t \rightarrow \RR^t$ which has depth $O(L)$ and width $O(tN)$ such that
    \begin{align*}
        \sup_{x\in[-1,1]^t}\| F(x) - \Psi(x) \|_{\infty} \leq C \lambda (N^2L^2 \log(N))^{-1/t},
    \end{align*}
    for a constant $C>0$ that depends only on $d$.
    Since any affine transformation from $\RR^N$ to $\RR^N$ can be represented by no more than $O(\log N) $ convolutional layers with a total number of parameters at most $N^2$ \cite{li2024approximation}, we can construct a convolutional neural network $\Phi$ with depth $O(L\log (tN))$ and total number of parameters $O(Lt^2N^2)$ such that $\Phi = \Psi$ and 
    \begin{align*}
        \sup_{x\in[-1,1]^t}\| F(x) - \Phi(x) \|_{\infty} \leq C \lambda (N^2L^2 \log(N))^{-1/t}.
    \end{align*}
    Since Lemma~\ref{lem:approximatorLip} implies that $\phi \circ V_m \circ \M \circ \phi^{-1}$ is a special case of $F$, we can conclude the proof by setting $\lambda := cL_{\M}m^d$.
\end{proof}


To approximate $\phi$ and $\phi^{-1}$, an error analysis can be conducted on a collection of sample points within $[-1,1]^d$ by employing the two lemmas outlined below.

\begin{lemma}\label{lem:encoder}
    There exists $\xi_{\en} = \{\xi_i\}_{i=1}^t \subset [-1,1]^d$ such that
    we can find a CNN $\Phi_{\en}$ with depth $O(\log t)$ and total number of parameters $O(t^2)$, satisfying
    \begin{align*}
        \Phi_{\en}(S(Q,\xi_{\en})) = \phi(Q), \quad \forall Q\in \Pi_m.
    \end{align*}
\end{lemma}
\begin{proof}
    Since $P_i \in \Pi_m$, $i=1,\dots,t$ are multidimensional Legendre polynomials which are constructed by the tensor products of univariate orthonormal Legendre polynomials with degree no more than $m$, by using \cite[Theorem 14.2.1]{davis1975interpolation}, there exists $\xi = \{\xi_1, \dots, \xi_t\}\subset [-1,1]^t$ and positive weights $w_i>0$, $i=1,\dots,t$ such that
    \begin{align*}
        \int_{[-1,1]^d} Q(x) P_i(x) dx = \sum_{j=1}^t w_j Q(\xi_j) P_i(\xi_j), \quad\forall Q\in \Pi_m.
    \end{align*}
    Let us define
    \begin{align*}
        U = 
        \begin{pmatrix}
            w_1 P_1(\xi_1) & w_2 P_1(\xi_2) & \dots & w_t P_1(\xi_t) \\
            \vdots & \ddots & \ddots  & \vdots \\
            w_1 P_i(\xi_1) & w_2 P_i(\xi_2) & \dots & w_t P_i(\xi_t) \\
            \vdots & \ddots & \ddots  & \vdots \\
            w_1 P_t(\xi_1) & w_2 P_t(\xi_2) & \dots & w_t P_t(\xi_t) 
        \end{pmatrix}.
    \end{align*}
    Then it is obvious that
    \begin{align*}
        U S(Q,\xi) = \phi(Q).
    \end{align*}
    Using \cite{li2024approximation}, there exists a CNN $\Phi_{\en}$ with depth $O(\log t)$ and total number of parameters $O(t^2)$ such that
    \begin{align*}
        \Phi_{\en}(S(Q,\xi)) = U S(Q,\xi).
    \end{align*}
\end{proof}

\begin{lemma}\label{lem:decoder}
    Given any $\xi_{\de} = \{\xi_i\}_{i=1}^t \subset [-1,1]^d$, there exists a CNN $\Phi_{\de}$ with depth $O(\log t)$ and total number of parameters $O(t^2)$, satisfying
    \begin{align*}
        \Phi_{\de}(\phi(Q))  = S(Q,\xi_{\de}), \forall Q\in \Pi_m.
    \end{align*}
    In addition, $\Phi_{\de}$ is Lipschitz, i.e.,
    \begin{align*}
        \| \Phi_{\de}(w ) - \Phi_{\de}(z ) \|_\infty \leq C t^2 \|w-z\|_\infty,
    \end{align*}
    for any $w,z \in \RR^t$.
\end{lemma}
\begin{proof}
    Define 
    \begin{align*}
        P = 
        \begin{pmatrix}
            P_1(\xi_1) &  P_2(\xi_1) & \dots &  P_t(\xi_1) \\
            \vdots & \ddots & \ddots  & \vdots \\
             P_1(\xi_i) &  P_2(\xi_i) & \dots &  P_t(\xi_i) \\
            \vdots & \ddots & \ddots  & \vdots \\
            P_1(\xi_t) &  P_2(\xi_t) & \dots &  P_t(\xi_t) 
        \end{pmatrix}.
    \end{align*}
    Then we can see that
    \begin{align*}
        S(Q,\xi) = \sum_{i=1}^t \langle Q, P_i \rangle S(P_i, \xi) = P\phi(Q).
    \end{align*}
    Using \cite{li2024approximation}, there exists a CNN $\Phi_{\de}$ with depth $O(\log t)$ and total number of parameters $O(t^2)$ such that
    \begin{align*}
        \Phi_{\de}(\phi(Q))  = S(Q,\xi) = P\phi(Q).
    \end{align*}

    Furthermore, since Lemma~\ref{lem:mhaskar} indicates that for any $j \in\{ 1,\dots,t \}$, $\|P_j\|_{L_\infty([-1,1]^d) }\leq C m^{d}$, we can obtain that $\Phi_{\de}$ is Lipschitz
    \begin{align*}
        \| \Phi_{\de}(w ) - \Phi_{\de}(z ) \|_\infty 
        &\leq t \sup_{j}\|P_j\|_{L_\infty} \|w-z\|_\infty\\
        &\leq C t^2 \|w-z\|_\infty,
    \end{align*}
    where we use the definition of $\infty$-matrix norm.
    
\end{proof}

We are now prepared to demonstrate the primary result.

\begin{theorem}
    Let $d, m, L, K \in \NN_{+}$, $t=(m+1)^d$ and $\Omega = [-1,1]^d$. Let $\X \subset \{ f\in\C(\Omega): \|f\|_{L_\infty} \leq 1\}$. Suppose Assumption~\ref{ass:Lipschitz} and Assumption~\ref{ass:Y} hold. Then, there exists a linear mapping $V_m :\C(\Omega)\rightarrow \Pi_m$ and a set of points $\xi_{\en} \subset \Omega$ with $|\xi_{\en}|=t$ such that for any $\xi_{\de} \subset \Omega$ with $|\xi_{\de}|=t$, we can find a U-shaped network $\Phi \in \unet\left(O( L\log (t K)), O(t^2 K)\right)$, satisfying 
    \begin{small}
        \begin{align}
        &\left\|S(\M(f),\xi_{\de}) - \Phi(S(V_m (f), \xi_{\en}))  \right\|_{\infty} \\
        &\leq C_1 \omega_{f}\left(\frac{2}{m} \right) + \frac{C_2}{m} +  \frac{C_3 t^{3} }{  ( L K \log(K/L))^{1/t}} ,
    \end{align}
    \end{small}
    for any $f\in \X$. Here, $C_1$,$C_2$,$C_3>0 $ depend on $d, L_{\M}, L_{\Y}$.
\end{theorem}

\begin{proof}
    Theorem~\ref{thm:discretization} implies that
    \begin{align}\label{eq:errorterm1}
    \begin{aligned}
        &\|S(\M(f)- V_m \circ \M \circ V_m (f) ,\xi_{\de})\|_{\infty } \\
        &\leq 6L_{\M}d^2 \omega_{f}\left(\frac{2}{m};[-1,1]^d \right) + \frac{5L_{\Y}d}{2m}.
    \end{aligned}
    \end{align}
    Applying Lemma~\ref{lem:encoder}, we get a CNN $\Phi_{\en}$ with depth $O(\log t)$ and total number of parameters $O(t^2)$ such that
    \begin{align}\label{eq:approximatorstep1}
    \begin{aligned}
        &V_m \circ \M \circ V_m (f)  \\
        &= \phi^{-1} \circ \left( \phi \circ V_m \circ \M \circ \phi^{-1} \right) \circ \phi \circ V_m (f) \\
        &= \phi^{-1} \circ \left( \phi \circ V_m \circ \M \circ \phi^{-1} \right) \circ \Phi_{\en}(S(V_m (f),\xi_{\en}))
    \end{aligned}
    \end{align}
    since $V_m (f) \in \Pi_m$.
    By the definition of $\phi$, we have
    \begin{align*}
        \|\phi \circ V_m (f)\|_\infty &\leq \| \phi \circ V_m (f) \|_2 \leq \| V_m (f) \|_{L_2} \\
        &\leq c \|V_m(f)\|_{L_\infty} 
        \leq c\|f\|_{L_\infty},
    \end{align*}
    where we use Lemma~\ref{lem:mhaskar} in the third step and Lemma~\ref{lem:linear_operator} in the last step. Hence we also have
    \begin{align}\label{eq:phien-norm}
        \|\Phi_{\en}(S(V_m (f),\xi_{\en}))\|_{\infty} \leq c \|f\|_{L_\infty} \leq c, \forall \|f\|_{\infty}\leq 1.
    \end{align}
    

    Lemma~\ref{lem:decoder} implies that we can find a CNN $\Phi_{\de}$ with depth $O(\log t)$ and total number of parameters $O(t^2)$ such that $\Phi_{\de}(\phi(Q)) = S(Q,\xi_{\de})$ for any $Q\in\Pi_m$. Substituting this fact into \eqref{eq:approximatorstep1} and using the definition of $\phi^{-1}$ yield
    \begin{align}\label{eq:Vmxi}
    \begin{aligned}
        &S(V_m \circ \M \circ V_m(f),\xi_{\de}) \\
        &= \Phi_{\de} \circ \left( \phi \circ V_m \circ \M \circ \phi^{-1} \right) \circ \Phi_{\en}(S(V_m (f),\xi_{\en})).
    \end{aligned}
    \end{align}
    
    The rest part is to construct a bottleneck $\Phi_{\app}$ for approximating $ \phi \circ V_m \circ \M \circ \phi^{-1}$.
    Lemma~\ref{lem:approximator} indicates that there exists a CNN $\Phi_{\app}$ with depth $O(L\log( tN))$ and total number of parameters $O(Lt^2N^2)$ such that
    \begin{align*}
        &\sup_{w\in[-1,1]^t} \| \phi \circ V_m \circ \M \circ \phi^{-1}(w) - \Phi_{\app}(w) \|_{\infty} \\
        &\leq C  L_{\M} m^{d} (N^2L^2 \log(N))^{-1/t}.
    \end{align*}
    Hence the following inequalities hold 
        \begin{align}\label{eq:errorterm2}
        \begin{aligned}
            &\| \Phi_{\de} \circ \left( \phi \circ V_m \circ \M \circ \phi^{-1} \right) \circ \Phi_{\en}(S(V_m (f),\xi_{\en})) \\
            &\quad- \Phi_{\de} \circ \Phi_{\app} \circ \Phi_{\en}(S(V_m (f),\xi_{\en})) \|_{\infty} \\
            &\leq  Ct^2 \sup_{w\in[-c,c]^t}  \| \phi \circ V_m \circ \M \circ \phi^{-1}(w) - \Phi_{\app}(w) \|_{\infty} \\
        &\leq Ct^2 L_{\M} m^{d} (N^2L^2 \log(N))^{-1/t}.
        \end{aligned}
    \end{align}
    where in the first step we use Lemma~\ref{lem:decoder} and in the second step we use Lemma~\ref{lem:approximator} and \eqref{eq:phien-norm}. Here we can apply Lemma~\ref{lem:approximator} since the corresponding proof can be easily adapted to domain $[-c,c]^t$.
    
    Let $\Phi := \Phi_{\de} \circ \Phi_{\app} \circ \Phi_{\en}$. Combining \eqref{eq:Vmxi}, \eqref{eq:errorterm1}, and \eqref{eq:errorterm2}, we obtain
    \begin{align*}
        &\left\|S(\M(f),\xi_{\de}) - \Phi(S(V_m (f), \xi_{\en}))  \right\|_{\infty} \\
        &\leq \|S(\M(f)- V_m \circ \M \circ V_m (f) ,\xi_{\de})\|_{\infty } \\
        &\quad + \| S(V_m \circ \M \circ V_m,\xi_{\de}) - \Phi(S(V_m (f),\xi_{\en})) \|_\infty \\
        &\leq 6L_{\M}d^2 \omega_{f}\left(\frac{2}{m};[-1,1]^d \right) + \frac{5L_{\Y}d}{2m} \\
        &\quad+ CL_{\M}\frac{t^{3} }{  (N^2L^2 \log(N))^{1/t}},
    \end{align*}
    where $\Phi$ is a CNN with depth $O(\log t + L\log (t N)) = O(  L\log (t N))$ and total number of parameters $O(t^2 + t^2 LN^2)=O( t^2 LN^2)$. Let $K = LN^2$, then we finish the proof.

\end{proof}

In the following, we offer a comprehensive analysis of the outcomes presented in the theorem.

\textbf{The role of $V_m$.} 
The linear operator $V_m$ transforms a continuous function into the space of polynomials. According to the Stone-Weierstrass theorem, polynomials are dense in $\C(\Omega)$. Therefore, substituting samples of $f$ with samples of $V_m(f)$ is adequate, provided that $m$ is sufficiently large.

\textbf{The role of $\xi_{\en}$.} 
In our analysis, we select a particular set of points to capture the information from a continuous function. Given that a finite number of sample locations can specify any polynomial in $\Pi_m$, sampling enough points allows us to completely represent $V_m(f)$.

\textbf{The role of $\xi_{\de}$.} 
U-shaped networks cannot generate functions as outputs, so we analyze the error between $\M(f)$ and $\Phi(S(V_m (f), \xi_{\en}))$ at the set of output locations $\xi_{\de}$. With $\xi_{\de}$ fixed, Theorem~\ref{thm:main} provides an effective network for these specific points. However, if the function's smoothness is inadequate, the same neural network might not perform well at other points.

\subsection{Proof of Theorem~\ref{thm:lowerbound}}
\label{proof:lowerbound}
We first need to establish the relationship between U-shaped networks and fully connected neural networks. Fully connected neural networks are defined as 
\begin{align*}
    \begin{aligned}
        \Psi = \A_{L} \circ \sigma \circ \A_{L-1} \circ \cdots \circ \sigma \circ \A_1,
    \end{aligned}
\end{align*}
where $\A_\ell(z)= A_\ell \cdot z + b_\ell$, involving a weight matrix $A_\ell$ and a bias vector $b_\ell$. We define $L$ as the number of layers within $\Psi$ and represent $K$ as the total number of non-zero parameters in the weight matrices and bias vectors of $\A_{\ell}$, $\ell\in[L]$. We denote $\nn(L,K)$ as the collection of all fully connected neural networks with no more than $L$ layers and at most $K$ nonzero parameters. Then we can show that U-shaped networks can be viewed as special fully connected neural networks.

\begin{lemma}\label{eq:unetFNN}
    Let $L,K\in\NN_{+}$ and $\unet(L,K), \nn(L,K)$ be the collection of U-shaped networks and fully connected networks defined on $\RR^t$. Then $\unet(L,K)\subset \nn(L,\max\{tK,K^2\})$.
\end{lemma}
\begin{proof}
    Let $x \in \RR^{n_1\times c_1}$ be a tensor with $c$ channels and each channel be denoted as $x_{i}\in \RR^{n_1}$. Given a set of kernels $w_{\ell,i} \in \RR^{2}$, $\ell \in [c_2]$, $i\in [c_1]$, the convolution between $x$ and these kernels is given by
    \begin{align*}
        y_\ell = \sum_{i=1}^{c_1} w_{\ell,i}*_{2} x_i = \sum_{i=1}^{c_1} T^{w_{\ell,i}} x_i,
    \end{align*}
    where $y\in \RR^{\lceil n_1/2 \rceil\times c_2}$ and $T^{w_{\ell,i}} \in \RR^{\lceil n_1/2 \rceil\times n_1}$ with at most $2\lceil n_1/2 \rceil \leq n_1+2$ nonzero entries.
    Denote
    \begin{align*}
        \tilde{x}:=
        \begin{pmatrix}
            x_1 \\
            \vdots \\
            x_{c_1}
        \end{pmatrix}
        ,\quad
        \tilde{y}:=
        \begin{pmatrix}
            y_1\\
            \vdots \\
            y_{c_2}
        \end{pmatrix}
    \end{align*}
    and
    \begin{align*}
        \tilde{T}:=
        \begin{pmatrix}
            T^{w_{1,1}} & T^{w_{1,2}} & \cdots & T^{w_{1,c_1}} \\
            T^{w_{2,1}} & T^{w_{2,2}} & \cdots & T^{w_{2,c_1}} \\
            \vdots      &  \vdots & \ddots & \vdots \\
            T^{w_{c_2,1}} & T^{w_{c_2,2}} & \cdots & T^{w_{c_2,c_1}} \\
        \end{pmatrix}
        .
    \end{align*}
    Obviously we have $\tilde{y} = \tilde{T}\tilde{x}$ and in $\tilde{T}$ there are at most $ (n_1+2) \times c_1 \times c_2$ parameters. 
    
    Meanwhile, there are $2c_1c_2$ parameters in the set of kernels $w_{\ell,i} \in \RR^{2}$, $\ell \in [c_2]$, $i\in [c_1]$. Thus, we can always rewrite any convolution operation into an affine mapping with nonzero parameters no more than the parameters in convolutional kernels times the dimension of each channel of inputs. In U-shaped networks, this is no greater than $\max\{t,K  \}$. If there is no flatten operation during the feed-forward processing, the stride-$2$ convolution with kernel size $2$ will always produce smaller dimensions in each channel. Hence, this might exceed $\max\{t,K\}$ only when there is a flatten operation. Before the flatten operation, we will get a tensor with each channel having dimension $1$. If this tensor has more than $K$ channels, the convolution layer to produce this tensor has at least $K$ convolution kernels which have $2K$ parameters. 
    However, $K$ represent the overall number of parameters in U-shaped networks. This presents a contradiction, suggesting that the output dimension of each layer does not exceed $\max\{t,K\}$. Consequently, the fully connected network formed has at most $\max\{tK,K^2\}$ parameters. Thus, the proof is complete.
\end{proof}

\begin{theorem}
    Let $d, m \in \NN_{+}$ and $t=(m+1)^d$. Let $\M:\Pi_m \rightarrow \Pi_m$ and $\X =  \Pi_m$. Suppose that Assumption~\ref{ass:Lipschitz} holds. Then there exists $\xi_{\en}, \xi_{\de} \subset [-1,1]^d$ with $|\xi_{\en}| = |\xi_{\de}|=t$ and a constant $C>0$ such that 
        \begin{align}
            \begin{aligned}
                &\inf_{\Phi\in \unet(L, K) } \sup_{f \in \X} \left\|S(\M(f),\xi_{\de}) - \Phi(S(f, \xi_{\en}))  \right\|_{\infty} \\
            &\geq C (tK^2L\log (tK^2))^{-1/t}.
            \end{aligned}
        \end{align}
\end{theorem}

\begin{proof}
    
    Since $\M:\Pi_m \rightarrow \Pi_m$, we can write
    \begin{align}\label{eq:Mpi}
        \M(f) = \phi^{-1} \circ \left( \phi \circ \M \circ \phi^{-1} \right) \circ \phi(f),
    \end{align}
    for any $f \in \Pi_m$. Using a similar approach as the proof of Lemma~\ref{lem:approximatorLip}, we can show that $\phi \circ \M \circ \phi^{-1}$ is Lipschitz. Lemma~\ref{lem:encoder} and Lemma~\ref{lem:decoder} implies that by setting $\xi_{\de} := \xi_{\en}$, for any $f \in \Pi_m$, there exist two matrices $P,U \in \RR^{t\times t}$ such that
    \begin{align*}
        U S(f,\xi_{\en}) &= \phi(f), \\
        S(f, \xi_{\de})  &= P \phi(f).
    \end{align*}
    Replacing these equalities into \eqref{eq:Mpi}, we can represent $\M$ over sample location set $\xi_{\de}$ as
    \begin{align*}
        S(\M(f),\xi_{\de}) = P \circ \left( \phi \circ \M \circ \phi^{-1} \right) \circ U\left( S(f,\xi_{\en}) \right).
    \end{align*}
    Let $F(\cdot) := P \circ \left( \phi \circ \M \circ \phi^{-1} \right) \circ U\left( \cdot \right)$. Obviously $F $ is Lipschitz. 
    
    We denote $\nn(L,K_2)$ with $K_2 := \max\{tK,K^2\}$.
    Using Lemma~\ref{eq:unetFNN}, we obtain
    \begin{align}\label{eq:unet2nn}
    \begin{aligned}
        &\inf_{\Phi\in \unet(L, K) } \sup_{f \in \X} \left\|S(\M(f),\xi_{\de}) - \Phi(S(f, \xi_{\en}))  \right\|_{\infty} \\
        &= \inf_{\Phi\in \unet(L, K) } \sup_{f \in \X} \left\|F\left(S(f,\xi_{\en}) \right) - \Phi(S (f, \xi_{\en}))  \right\|_{\infty} \\
        &\geq \inf_{\Phi\in \nn(L, K_2) } \sup_{x\in[-1,1]^d} \left\|F(x) - \Phi(x)  \right\|_{\infty}\\
        &\geq \inf_{\Phi\in \nn(L, K_2) } \left\|F_1(x) - \Phi_1(x)   \right\|_{L_\infty} ,
    \end{aligned}
    \end{align}
    where we use the notation $F(\cdot) := (F_1(\cdot),\dots, F_t(\cdot))^\top$ and $\Phi(\cdot) := (\Phi_1(\cdot),\dots, \Phi_t(\cdot))^\top$.

    Since $F_1$ is a Lipschitz function, we can use 
    \cite[Theorem 2.4]{shen2022optimal} and \cite{bartlett2019nearly} to get the lower bound when using fully connected neural networks for approximation
    \begin{align}\label{eq:nnlower}
        \inf_{\Phi\in \nn(L, K_2) } \left\|F_1(x) - \Phi_1(x)   \right\|_{L_\infty} \geq C (K_2L\log K_2)^{-1/t},
    \end{align}
    where $C$ only depends on $t$. 
    Substituting \eqref{eq:nnlower} into \eqref{eq:unet2nn}, we finally obtain
    \begin{align*}
        &\inf_{\Phi\in \unet(L, K) } \sup_{f \in \X} \left\|S(\M(f),\xi_{\de}) - \Phi(S(f, \xi_{\en}))  \right\|_{\infty} \\
        &\geq C (tK^2L\log (tK^2))^{-1/t}.
    \end{align*}
    We complete the proof.
\end{proof}

Theorem~\ref{thm:lowerbound} suggests that even if we restrict $\M$ to act as a mapping between finite-dimensional subsets of $\C(\Omega)$, we cannot expect a significant enhancement. Once the network architecture is fixed, there will always be a lower bound.

Theorem~\ref{thm:main} and Theorem~\ref{thm:lowerbound} jointly offer a significant understanding of $B_\alpha(\Phi(\cdot))$. According to Definition~\ref{def:hallucination}, hallucination occurs when $\M(X) \notin B_\alpha(\Phi(X))$.
Our theoretical findings indicate several factors that could lead to hallucination: (i) the intrinsic natures of real world scenes; (ii) the discretization process; (iii) the expressivity of the network. These factors also indicate that hallucination might be a universal problem in machine learning.

\begin{Remark}
\textbf{Theoretical Perspectives on U-Shaped Networks:}
While recent research has aimed to improve the performance of U-Nets, a complete theoretical explanation of their effectiveness remains elusive. Recent efforts to clarify its representational power have focused on interpretative views of the architecture. The learning task has been formulated as a PDE-constrained optimal control problem, where the U-Net corresponds to a single iteration of a hybrid operator-splitting scheme \cite{tai_mathematical_2024}. Another interpretation shows that the U-Net structure arises from the decomposition of a Hankel matrix associated with sparse-view CT data, establishing a connection between deep learning and compressed sensing via deep convolutional framelets \cite{han_framing_2018}. It has also been shown that U-Net’s architecture can implement the belief propagation algorithm on generative hierarchical models, i.e. tree-structured graphical models, thereby efficiently approximating denoising functions \cite{mei_u-nets_2024}. Finally, a general framework for designing and analyzing U-Net architectures highlights the distinct roles of the encoder and decoder \cite{williams_unied_2023}.
\end{Remark}

\end{document}